%% file: main-arxiv.tex
\documentclass[
  bibliography=totoc,
  DIV=12,
]{scrartcl}

\usepackage[utf8]{inputenc}
\usepackage[T1]{fontenc}

\usepackage{nat_preamble}

\newtheorem{theorem}{Theorem}[section]
\newtheorem{corollary}{Corollary}[section]
\newtheorem{lemma}{Lemma}[section]
\newtheorem{proposition}{Proposition}[section]
\newtheorem{definition}{Definition}[section]

\newtheorem{remark}{Remark}[section]

\setkomafont{disposition}{\normalfont\bfseries}

\title{%
  Sample Path Regularity of Gaussian Processes from the Covariance Kernel
}
\author{%
  Nathaël Da Costa\thanks{Tübingen AI Center, University of Tübingen}
  \and
  Marvin Pförtner\footnotemark[1]
  \and
  Lancelot Da Costa\thanks{Verses AI Research Lab \& ELLIS Institute Tübingen}
  \and
  Philipp Hennig\footnotemark[1]
}

\begin{document}

\maketitle

\begin{abstract}
\input{abstract}
\end{abstract}

\input{content}

\newpage

\appendix

\section*{Appendix}

\input{appendix}

\bibliographystyle{abbrv}
\bibliography{references}

\end{document}

%% file: abstract.tex
Gaussian processes (GPs) are the most common formalism for defining probability distributions over spaces of functions. While applications of GPs are myriad, a comprehensive understanding of GP sample paths, i.e.~the function spaces over which they define a probability measure, is lacking. In practice, GPs are not constructed through a probability measure, but instead through a mean function and a covariance kernel. In this paper we provide necessary and sufficient conditions on the covariance kernel for the sample paths of the corresponding GP to attain a given regularity. We focus primarily on Hölder regularity as it grants particularly straightforward conditions, which simplify further in the cases of stationary and isotropic GPs. We then demonstrate that our results allow for novel and unusually tight characterisations of the sample path regularities of the GPs commonly used in machine learning applications, such as the Matérn GPs.

%% file: content.tex
\section{Introduction}
Gaussian processes (GPs) provide a formalism to assign probability distributions over spaces of functions. That distribution is principally characterised and controlled by the process' covariance function, which is a positive definite kernel. Such kernels are also associated with reproducing kernel Hilbert spaces (RKHSs).
However, it is relatively widely known that the \emph{sample paths} of the associated GP are not generally elements of the RKHS, but form a ``larger'' space of typically less regular functions \cite[Section 4]{kanagawa_gaussian_2018}. This sample path space is \emph{much harder to characterise} than the RKHS.

GP regression is widely applied in statistics and machine learning for inference from physical observations. In such settings, practitioners mostly concern themselves with the posterior mean function (which \emph{is} an element of the RKHS) and the marginal variance (which is \emph{not} in the RKHS, but inherits its regularity) while largely ignoring the regularity of the sample paths and other properties of the support of the prior probability measure. But recently relevant use cases for GPs in computational tasks urgently require a more careful, and ideally tight analysis of the sample path regularity, which we provide in this work. For instance, \cite{pfortner_physics-informed_2023} recently showed that a large class of classic solution methods for linear partial differential equations (PDEs) can be interpreted as GP inference, i.e.~as base instances of probabilistic numerical methods. To infer the solution to a PDE using a probabilistic numerical method, we want to construct a GP prior for it. It should be tailored to the problem as tightly as possible, for the following reasons:
\begin{enumerate}
    \item We need to ensure that the PDE's differential operator is well-defined on all samples of the GP. Hence, the sample paths must be \textbf{regular enough}. Here, regularity typically refers to the existence of a number of strong or weak partial derivatives, which is encapsulated in the frameworks of Hölder and Sobolev regularity respectively.
    \item But we also want the GP posterior to provide a useful uncertainty estimate over the solution of the PDE. Hence, we do not want to needlessly impose \emph{additional} regularity constraints on the sample paths. The sample paths should be \textbf{as irregular as possible}, to avoid overconfidence. Such cautious models may also be advisable on numerical grounds, to avoid instabilities such as Gibbs or Runge phenomena.
\end{enumerate}
The above is just one example of a setting in which one would like to characterise the regularity of GP sample paths as tightly as possible. Further examples include generalised coordinates, a method that underlies the analysis and filtering of stochastic differential equations driven by noise admitting a high number of derivatives \cite{heins_collective_2023}. These noise signals are usually modelled with GPs, and hence characterising GP sample path regularity is crucial for understanding in which situations these methods can be applied \cite{costa_theory_2024}.
\begin{figure}[t]\label{fig: tensor materns}
\centering
\includegraphics[height=174pt]{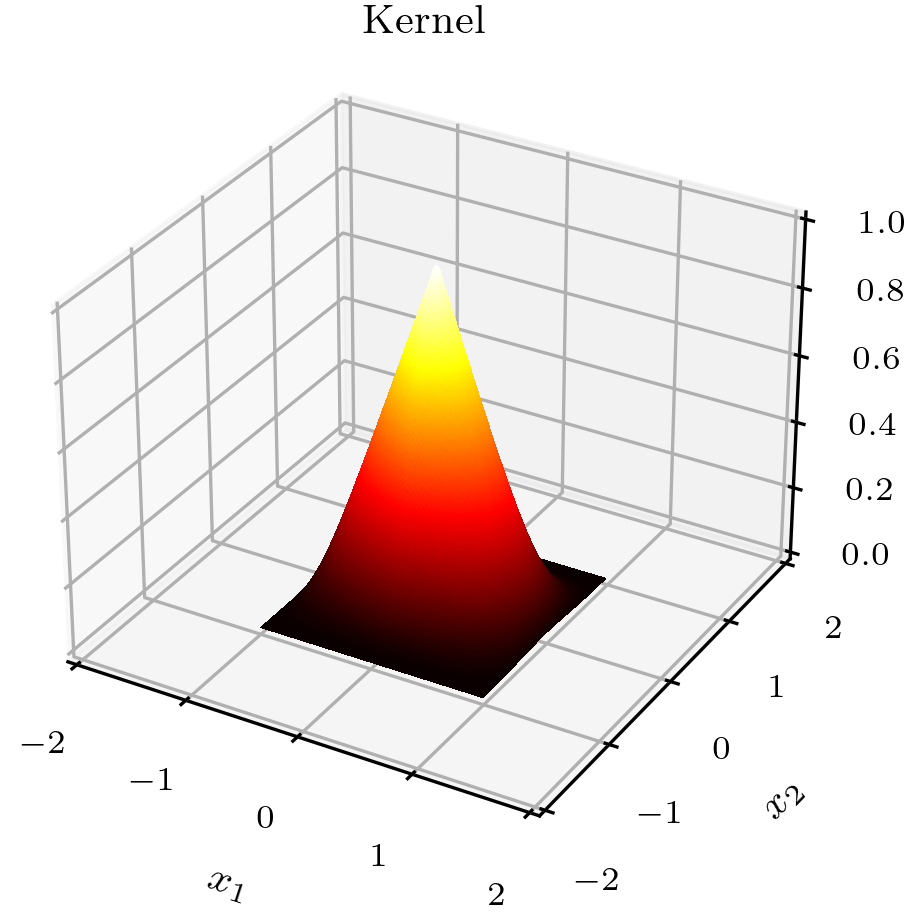}
\includegraphics[height=174pt]{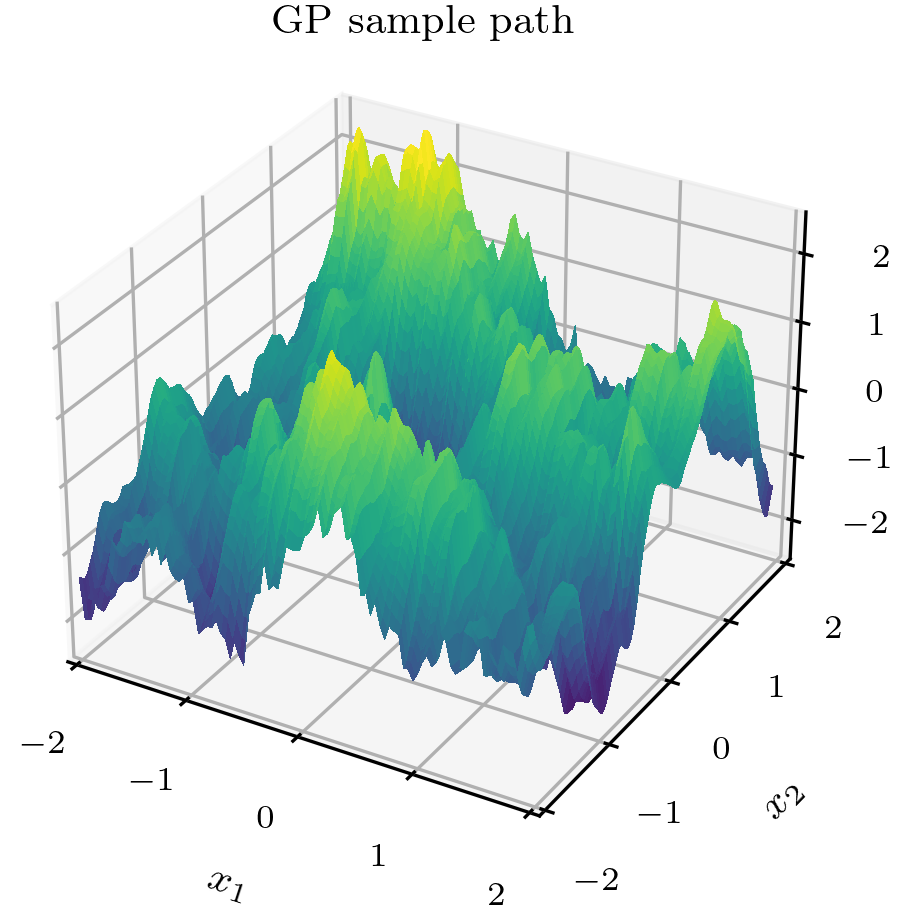}
\caption{\textit{Left:} surface plot of a kernel defined as the tensor product of two Wendland kernels (with $d=1$, $n=0$ on the $x_1$-axis and $d=1$, $n=1$ on the $x_2$-axis). \textit{Right:} surface plot of a sample path from the corresponding centered GP. The sample path has different regularity along each axis (non-differentiable along the $x_1$-axis; once differentiable along the $x_2$-axis), reflecting the rotational asymmetry in the regularity of the kernel.}
\end{figure}

\subsection{Summary of Contributions}

A consequence of our main result, Theorem \ref{thm: holder}, may be written as follows:
\begin{corollary}
    Let $k \colon \R^d \times \R^d \to \R$ be a symmetric positive definite kernel.
    If either
    \begin{itemize}
        \item all partial derivatives of the form $\frac{\partial^{2n} k}{\partial \bm x^{\bm \alpha}\partial \bm y^{\bm \alpha}}$ for multi-indices $\bm \alpha \in \N_0^d$ with $|\bm\alpha| = n$ exist and are (locally) Lipschitz,
        \item for stationary $k(\bm x, \bm y) = k_\delta(\bm x - \bm y)$, $\frac{\partial^{2n} k_\delta}{\partial \bm x^{\bm\alpha}}$ for $|\bm \alpha| = 2n$ exist and are (locally) Lipschitz,
        \item for isotropic $k(\bm x, \bm y) = k_r(\| \bm x - \bm y \|)$, $k_r^{(2n)}$ exists and is (locally) Lipschitz,
    \end{itemize}
    then the sample paths of $f \sim \GP(0, k)$ are, up to modification, $n$ times continuously differentiable.
\end{corollary}
Note that, by the mean value theorem, the existence of one additional continuous derivative of $k$ is sufficient for local Lipschitz continuity of the lower derivatives.

The actual Theorem \ref{thm: holder} is sharper than this.
It provides necessary and sufficient conditions on the regularity of the kernel for the sample paths to attain a given Hölder regularity.
Just as the above corollary, the theorem encompasses statements for stationary and isotropic GPs.

We apply Theorem \ref{thm: holder} in particular to the Matérn GPs to obtain Proposition \ref{prop: matern}, a consequence of which is the following:
\begin{corollary}
    The sample paths of a centered Matérn GP with smoothness parameter $\nu\not\in\N_0$ are, up to modification, $\lfloor\nu\rfloor$ times continuously differentiable and no more.
    In particular, when $\nu =n+\nicefrac{1}{2}$ for some $n\in\N_0$, the sample paths are, up to modification, $n$ times continuously differentiable and no more.
\end{corollary}
We then obtain similar results for other GPs, such as the Wendland (Proposition \ref{prop: wendland}), squared exponential and rational quadratic GPs (Remark \ref{rmk: smooth}).

In Section \ref{sec: kernel algebra} we describe how to apply Theorem \ref{thm: holder} to covariance kernels constructed through algebraic transformations of other kernels, including conic combinations (Proposition \ref{prop: conic}), products (Proposition \ref{prop: product}), tensor products (Proposition \ref{prop: tensor product}) and coordinate transformations (Proposition \ref{prop: coord}). We note in Section \ref{sec: manifold} how the results in this work may be generalised to manifold GPs.

Theorem \ref{thm: sobolev} summarises results about GP sample path Sobolev regularity. Comparing these with Theorem \ref{thm: holder}, we see that for stationary GPs one should not expect sample paths to admit a greater number of weak (Sobolev) derivatives than of strong derivatives.

We then plot a number of stationary kernel functions as well as sample paths from the corresponding centered GPs, as in Figure \ref{fig: tensor materns}.

All regularity properties considered in this work are \emph{local}, in the sense that they capture the infinitesimal behaviour of functions around each point in their domain. Note that our proofs can be extended and our assumptions strengthened to consider global regularity properties, such as global Hölder continuity.

\subsection{Related Work}
Some of the earliest work on sample path regularity can be attributed to Kolmogorov and \cite{chentsov_weak_1956}. Later, \cite{fernique_regularite_1975} proved the first necessary and sufficient conditions for sample path continuity of stationary GPs. Since then, many authors have extended these continuity results. \cite{potthoff_sample_2009} provides an approachable overview of how such results can be obtained. It also investigates conditions for uniform continuity, which is a \emph{global} regularity property. \cite[Theorem 8.3.2]{adler_geometry_2010} characterises the Hölder continuity of GPs. \cite{azmoodeh_necessary_2014} further provides necessary and sufficient conditions for Hölder continuity of GPs with one-dimensional real inputs, which can be viewed as a converse to the Kolmogorov continuity theorem for GPs. \cite{nummi_necessary_2024}, which we rely on in this work, generalises this to multi-dimensional inputs, and to a class of hypercontractive processes.

Differentiability has also been investigated by many authors, including \cite{scheuerer_regularity_2010} for general random fields, \cite[Theorem 1.4.2]{adler_random_2007}, \cite{potthoff_sample_2010} and \cite{henderson_sobolev_2024} for Sobolev regularity. \cite[Proposition I.3]{ghosal_fundamentals_2017} combines differentiability results with the Kolmogorov continuity theorem, and can be seen as the global version of the forward implication of our \thmitemref{Theorem}{thm: holder}{thm holder: general case}. \cite{kerkyacharian_regularity_2018} gives necessary and sufficient conditions for the Hölder regularity (both differentiability and Hölder continuity) of GPs on certain compact metric spaces, obtaining a result which resembles Theorem \ref{thm: holder}. However, unlike Theorem \ref{thm: holder} and as discussed by the authors, its assumptions are difficult to verify.

Although no work applies general GP sample path results across a range of examples, some particular GPs have been investigated individually in the literature. The regularity of Brownian motion sample paths is well-studied (see for example \cite{roynette_mouvement_1993}). The Matérn GPs were studied in \cite[Example 4.8]{steinwart_convergence_2019} and \cite[Examples 5.3.19, 5.5.6 \& 5.5.12]{scheuerer_comparison_2010}.

Compared to the existing literature, the present work offers a unique combination of \emph{generality}, \emph{versatility}, and \emph{practical applicability} in the study of GP sample path continuity and differentiability. To achieve this we unify existing results and prove several new ones. Novel results include the converse of Theorem \ref{thm: holder} relating the regularity of the sample paths to the regularity of the kernels, \thmitemref{Theorem}{thm: holder}{thm holder: isotropic case} characterising the sample path regularity for GPs with isotropic kernels, Proposition \ref{prop: product} characterising it for GPs with product kernels, and Proposition \ref{prop: tensor product} characterising it for GPs with tensor product kernels.
\section{Preliminaries}\label{sec: preliminaries}
\begin{definition}\label{def: gp}
    A Gaussian process on a set $O$ is a map $f \colon O\times \Omega \to \R$, where $\Omega$ is a probability space, such that for all $x\in O$, $f(x,\cdot) \colon O \to \R$ is measurable, and such that for each $X := (x_1,\dots,x_N) \in O^N$, the map $f(X, \cdot) \colon \Omega \to \R^{N}$ given by $f(X, \omega) =(f(x_1,\omega),\dots, f(x_N,\omega))$ is a multivariate Gaussian random variable.

    The maps $f(\cdot,\omega) \colon O \to \R$ for $\omega\in \Omega$ are the sample paths of the GP.
\end{definition}
We would like to study the regularity of the sample paths of a GP. However, Definition \ref{def: gp} is not convenient to work with in practice: one rarely constructs the probability space $\Omega$ of a GP. Instead, one characterises a GP by its \emph{mean} and its \emph{covariance kernel}.
\begin{definition}\label{def: mean and kernel}
    The mean $\mu$ of a GP $f \colon O\times \Omega \to \R$ is the map
    $$\mu \colon O \to \R,\; x \mapsto \E(f(x,\cdot)).$$
    The covariance kernel $k$ is the map
    $$k \colon O \times O \to \R,\; (x,y) \mapsto \cov(f(x,\cdot),f(y,\cdot)).$$
    The GP $f$ is said to be centered if $\mu = 0$.
\end{definition}
Conversely, given maps $\mu \colon O \to \R$ and $k \colon O \times O \to \R$ with $k$ symmetric positive definite -- meaning that for any finite set of points $\{x_1,\dots,x_N\} \subset \Omega$, the matrix $(k(x_i,x_j))_{i,j}$ is symmetric positive semi-definite -- one can construct a probability space $\Omega$, and a GP $f \colon O\times \Omega \to \R$ with mean $\mu$ and covariance kernel $k$ \cite[Theorem 14.36]{klenke_probability_2014}. We then write $f\sim \GP(\mu,k)$.

At this point it is important to note that $f\sim \GP(\mu,k)$ does not uniquely specify $f$ or the probability space $\Omega$. However it uniquely specifies the finite dimensional distributions of $f$ -- since multivariate Gaussians are entirely characterised by their first two moments. In machine learning applications all that will ever be observed are evaluations of $f$ at finitely many points. Therefore we would like to study sample path regularity of $f$ \emph{up to modification}. By this we mean that $f\sim \GP(\mu,k)$ will be said to have samples in $\mathcal F(O)$, where $\mathcal F(O)$ is some space of functions $O \to \R$, if there exists a construction of the GP $f$, say $f \colon O \times \Omega \to \R$, such that $f(\cdot,\omega)\in \mathcal F (O)$ for all $\omega\in \Omega$.

Finally, note that $f\sim \GP(\mu,k)$ having samples in $\mathcal F(O)$ is equivalent to $\tilde f+\mu$ having samples in $\mathcal F(O)$, where $\tilde f\sim \GP(0,k)$. For this it suffices that $\mu\in \mathcal F(O)$ and $\tilde f$ has samples in $\mathcal F(O)$. Therefore, in what follows, we will consider only centered GPs.

Thus we are studying sample path regularity of GPs \emph{from the covariance kernel}, and our goal is to link the regularity of the kernel to the regularity of the GP sample paths.

\subsection{Setup}
We study GPs defined on open subsets of Euclidean spaces: $O\subset \R^d$ is open, for some $d\in\N$, and $f\sim \GP(0, k)$ is a centered GP on $O$.
An important special case is when $k$ is stationary, i.e.~$O=\R^d$ and there is an even function $k_\delta \colon \R^d \to \R$ such that
\begin{equation}
    k(\bm x, \bm y) = k_\delta (\bm x-\bm y)
\end{equation}
for all $\bm x, \bm y \in \R^d$.
A further special case of $k$ stationary is when $k$ is isotropic, i.e.~$O=\R^d$ and there is an even function $k_r \colon \R \to \R$ such that
\begin{equation}
    k(\bm x, \bm y) = k_r (\|\bm x-\bm y\|)
\end{equation}
for all $\bm x, \bm y \in \R^d$, where $\|\cdot \|$ denotes the standard Euclidean norm on $\R^d$.

\section{Hölder Regularity}
The main type of sample path regularity we would like to investigate is continuity and the order of continuous differentiability. These are encapsulated in the framework of \emph{Hölder regularity}. \par
We choose $O$ to be an open set in $\R^d$ in order to capture differentiability of the sample paths. However this implies that Hölder spaces on $O$ are inconvenient for our purposes, as they do not solely characterise local regularity properties of functions on $O$. Indeed, the Hölder spaces not only constrain the local behaviour of functions, but also their behaviour at infinity, or near the boundary of $O$. To retain solely local constraints we employ local Hölder spaces.
\begin{definition}[Local Hölder and almost-Hölder spaces]\label{def: holder}
    Let $n \in \N_0$ and $\gamma \in [0,1]$.
    \begin{enumerate}[label=(\arabic*)]
        \item The local Hölder space $C^{n,\gamma}_{\loc}(O)$ is the space of functions $f$ on $O$ for which $\partial^{\bm \alpha} f$ exists for all multi-indices $\bm \alpha = (\alpha_1,\dots,\alpha_d)\in \N_0^d$ with $|\bm\alpha| := \alpha_1+\dots+\alpha_d\leq n$, and such that the highest order partial derivatives satisfy a Hölder condition of the form: for all compact subsets $K \subset O$ there is a constant $C_K>0$ such that
            $$|\partial ^{\bm \alpha} f(\bm x) -\partial ^{\bm \alpha} f(\bm y)| \leq C_K\|\bm x-\bm y\|^\gamma$$
            for all $\bm x,\bm y \in K$ and $|\bm\alpha|=n$.
    \item The local almost-Hölder space $C^{(n+\gamma)^-}_{\loc}(O)$ is defined as $\bigcap_{n'+\gamma'<n+\gamma}C^{n',\gamma'}_{\loc}(O)$.
    \end{enumerate}
\end{definition}
\begin{remark}
    For $n,n'\in \N_0$ and $\gamma,\gamma'\in [0,1]$, $n'+\gamma' < n +\gamma$ implies $C^{n',\gamma'}_{\loc}(O)\supsetneq C^{(n+\gamma)^-}_{\loc}(O)\supsetneq C^{n,\gamma}_{\loc}(O)$. Moreover $C^n(O)=C^{n,0}_{\loc}(O)$.

    The key point in the definition of $C^{(n+\gamma)^-}_{\loc}(O)$ is the strict inequality $n'+\gamma'<n+\gamma$, since allowing for equality we have $\bigcap_{n'+\gamma'\leq n+\gamma}C^{n',\gamma'}_{\loc}(O)=C^{n,\gamma}_{\loc}(O)$.
\end{remark}
As will become clear, the local almost-Hölder spaces are natural for characterising GP sample path regularity.

To describe the regularity of the covariance kernel $k$, we will require derivatives of the form $\partial^{\bm \alpha,\bm \beta}k$. Here $\partial^{\bm \alpha,\bm \beta}$ stands for the application of $\partial^{\bm\alpha}$ with respect to the first variable, followed by the application of $\partial^{\bm\beta}$ with respect to the second variable. In fact the order in which they are applied will not matter, by continuity of these partial derivatives. We moreover write $C^{n\otimes n}(O\times O)$ for the space of functions $k$ on $O\times O$ for which $\partial^{\bm \alpha,\bm \beta}k$ exists and is continuous for all $|\bm\alpha|,|\bm\beta|\leq n$.

Also recall the ``big O'' notation: the functions $f,g\colon\R^d\setminus \{\bm 0\}\to \R$ satisfy $f(\bm h) = \mathcal O(g(\bm h))$ as $\bm h\to \bm 0$ if and only if there is $C>0$ such that $\limsup_{\bm h\to\bm 0} \left|f(\bm h)/g(\bm h)\right| \leq C$. Moreover, for a family of functions $(f_{\bm x} )_{\bm x\in O}$, we say $f_{\bm x}(\bm h) = \mathcal O(g(\bm h))$ as $\bm h\to \bm 0$ locally uniformly in $\bm x\in O$ if, for every compact subset $K\subset O$, $C>0$ may be chosen independently of $\bm x\in K$.

We are now in a position to state the main result of this paper.
\begin{restatable}[Sample path Hölder regularity]{theorem}{thmholder}
\label{thm: holder}
Let $n\in\N_0$ and $\gamma\in (0,1]$.
The process $f \sim \GP(0, k)$ has samples in $C^{(n+\gamma)^-}_{\loc}(O)$ if and only if,
\begin{enumerate}[label=(\arabic*)]
    \item \label{thm holder: general case} for general $k$,
        \begin{itemize}
            \item $k\in C^{n\otimes n}(O\times O)$,
            \item $|\partial^{\bm \alpha,\bm \beta} k(\bm x+\bm h,\bm x+\bm h)-\partial^{\bm \alpha,\bm \beta}k(\bm x+\bm h,\bm x)-\partial^{\bm \alpha,\bm\beta}k(\bm x,\bm x+\bm h)+\partial^{\bm \alpha,\bm \beta}k(\bm x,\bm x)| = \mathcal{O}(\|\bm h\|^{2\epsilon})$ \\
            as $\bm h \to \bm 0$, locally uniformly in $\bm x\in O$, for all $\epsilon \in (0,\gamma)$ and $|\bm \alpha|=|\bm\beta|=n$.
        \end{itemize}
    \item \label{thm holder: stationary case} for stationary $k(\bm x, \bm y) = k_\delta(\bm x - \bm y)$,
        \begin{itemize}
            \item $k_\delta\in C^{2n}(\R^d)$,
            \item $|\partial^{\bm\alpha}k_\delta(\bm h)-\partial^{\bm\alpha}k_\delta(\bm 0)| = \mathcal{O}(\|\bm h\|^{2\epsilon})$ as $\bm h \to \bm 0$ for all $\epsilon \in(0,\gamma)$ and $|\bm \alpha|=2n$.
        \end{itemize}
    \item \label{thm holder: isotropic case} for isotropic $k(\bm x, \bm y) = k_r(\| \bm x - \bm y \|)$,
        \begin{itemize}
            \item $k_r\in C^{2n}(\R)$,
            \item $|k_r^{(2n)}(h)-k_r^{(2n)}(0)| = \mathcal{O}(|h|^{2\epsilon})$ as $h \to 0$ for all $\epsilon \in (0, \gamma)$.
        \end{itemize}
\end{enumerate}
In each case, differentiating sample-wise we have $\partial^{\bm \alpha} f\sim \GP(0, \partial^{\bm \alpha,\bm\alpha} k)$ for all $|\bm\alpha|\leq n$.
\end{restatable}
\noindent The \hyperref[proof: mainthm]{proof} can be found in \hyperref[app: holder thm]{Appendix A}.
\begin{remark}\label{rmk: kernel regularity}
    The positive definiteness of $k$ allows us to extrapolate regularity: $k$, $k_\delta$ or $k_r$ have no more derivatives around the diagonal $\{(\bm x,\bm x)\in O\times O: \bm x\in O\}$, $\bm 0$ or $0$ respectively as they do on the rest of the domain (see Lemma \ref{lem extrapol}). Therefore, when applying Theorem \ref{thm: holder}, it suffices to check the existence of the respective derivatives around the diagonal, $\bm 0$ or $0$ respectively.

    Another way to see this is that, to characterise any \emph{local} sample path regularity property, information about $k$ in a neighbourhood of the diagonal must be sufficient as it must suffice to consider the covariances at arbitrarily close by points.
\end{remark}
Theorem \ref{thm: holder} combines results about sample Hölder continuity of GPs \cite{azmoodeh_necessary_2014} (or the more recent \cite{nummi_necessary_2024} for the multi-dimensional case) to results about sample differentiability of GPs \cite{potthoff_sample_2010}, applied inductively on the partial derivatives. It provides necessary and sufficient conditions for $f$ to have almost-Hölder continuous samples of a certain degree. Let us note however that it does not quite give us necessary conditions for $f$ to have samples in $C^n(O)$, as the latter is not an almost-Hölder space. To achieve this, we can adapt the proof of Theorem \ref{thm: holder} combining necessary and sufficient conditions for sample continuity of GPs with the results about sample differentiability from \cite{potthoff_sample_2010} (see \cite[Theorem 5.3.16]{scheuerer_comparison_2010} for a result of this flavour). Characterising necessary and sufficient conditions on the kernel for sample continuity of the GP is much more involved than almost-Hölder continuity, and therefore we do not expand our results to this setting (see \cite{fernique_regularite_1975} for the stationary case, \cite{talagrand_regularity_1987} for the general case). Moreover we demonstrate in the examples below that Theorem \ref{thm: holder} is all we need.

In the following we demonstrate how Theorem \ref{thm: holder} can be applied in practice.
We investigate the sample path regularity of the most widely used GP families, recovering known results as well as proving novel ones.

The reverse implications in Theorem \ref{thm: holder} allow us to obtain sharp sample path regularity characterisations in the following sense: $f$ is said to have samples in $C^{(n+\gamma)^-}_{\loc}(O)$ \emph{and no more}, if $f$ is not sample $C^{(n'+\gamma')^-}_{\loc}(O)$ for any $n'+\gamma'>n+\gamma$.

\subsection{Wiener Kernel}
The Wiener process is a centered GP on  $O = \R_{>0}$ with covariance kernel
$$k(x,y) = \min(x,y)$$
for $x,y\in \R_{>0}$. $\min(\cdot,\cdot)$ is Lipschitz but non-differentiable in both its arguments. So by \thmitemref{Theorem}{thm: holder}{thm holder: general case} the Wiener process has samples in $C^{\nicefrac{1}{2}^-}_{\loc}(\R_{>0})$ and no more.
\begin{remark}
    By the fundamental theorem of calculus, the $n$-times integrated Wiener process \cite[Section B]{schober_probabilistic_ode_2014} has samples in $C^{(n+\nicefrac{1}{2})^-}_{\loc}(\R_{> 0})$ and no more.
\end{remark}

\subsection{Matérn Kernels}
The Matérn kernels are isotropic kernels on $\R^d$ given by
\begin{equation}\label{eq: matern kernel}
k(\bm x,\bm y) = k_r(\|\bm x-\bm y\|) = \frac{2^{1-\nu}}{\Gamma(\nu)}\left(\sqrt{2\nu}\|\bm x-\bm y\|\right)^\nu K_\nu\left(\sqrt{2\nu}\|\bm x-\bm y\|\right)
\end{equation}
for $\bm x,\bm y\in\R^d$, where $K_\nu$ is the modified Bessel function of the second kind and $\nu >0$ is the smoothness parameter. The cases where $\nu = n+\nicefrac{1}{2}$ for some $n\in\N_0$ are particularly interesting because the expression in Equation (\ref{eq: matern kernel}) simplifies to a product of a polynomial and an exponential of the radial distance \cite[Equation 4.16]{rasmussen_gaussian_2005}.
\begin{proposition}\label{prop: matern}
    A centered Matérn GP with smoothness parameter $\nu>0$ has samples in $C^{\nu^-}_{\loc}(O)$ and no more. In particular, if $\nu = n+\nicefrac{1}{2}$ for some $n\in\N_0$, the GP has samples in $C^{(n+\nicefrac{1}{2})^-}_{\loc}(\R^d)$ and no more.
\end{proposition}
This follows from \thmitemref{Theorem}{thm: holder}{thm holder: isotropic case} and the proof can be found in \hyperref[app: proofs]{Appendix B} \ref{proof: matern}.
\begin{remark}
    Proposition \ref{prop: matern} also covers the (multivariate) Ornstein-Uhlenbeck process, the centered GP with covariance kernel given by $k(\bm x,\bm y) = \exp(-\|\bm x-\bm y\|)$ which corresponds to the Matérn GP with $\nu =\nicefrac{1}{2}$. So the (multivariate) Ornstein-Uhlenbeck process has samples in $C^{\nicefrac{1}{2}^-}_{\loc}(\R^d)$ and no more.
\end{remark}

\subsection{Wendland Kernels}
The Wendland kernels are isotropic kernels which are compactly supported and piecewise polynomial in the radial distance. On $\R^d$, they are defined as follows \cite[Definition 9.11]{wendland_scattered_2004}:
$$k(\bm x,\bm y) = k_r(\|\bm x-\bm y\|) = \mathcal I^n \phi_{\lfloor \nicefrac{d}{2} \rfloor+n+1}(\|\bm x-\bm y\|)$$
for $\bm x,\bm y\in\R^d$, where $\phi_j(\rho) := \max((1-\rho)^j,0)$ and $\mathcal I\phi (\rho) := \int_\rho^\infty t\phi(t)\,dt/ \int_0^\infty t\phi(t)\,dt$ for $\rho\geq0$. $n\in \N_0$ is a parameter controlling the degree of the polynomial (precisely, $k_r(x)$ has degree $\lfloor \nicefrac{d}{2}\rfloor +3n+1$ in $\rho$ around 0).
\begin{proposition}\label{prop: wendland}
    \sloppy A centered Wendland GP with degree parameter $n$ has samples in $C^{(n+\nicefrac{1}{2})^-}_{\loc}(O)$ and no more.
\end{proposition}
This follows from \thmitemref{Theorem}{thm: holder}{thm holder: isotropic case} and the proof can be found in \hyperref[app: proofs]{Appendix B} \ref{proof: wendland}.
\subsection{Smooth Kernels}
From Theorem \ref{thm: holder} immediately follows the following important corollary:
\begin{corollary}
The process $f\sim \GP(0,k)$ has samples in $C^\infty(O)$ if and only if,
    \begin{enumerate}[label=(\arabic*)]
        \item for general $k$, $k\in C^\infty(O\times O)$.
        \item for stationary $k(\bm x,\bm y)=k_\delta(\bm x-\bm y)$, $k_\delta\in C^\infty (\R^d)$.
        \item for isotropic $k(\bm x,\bm y)=k_r(\|\bm x-\bm y\|)$, $k_r\in C^\infty (\R)$.
    \end{enumerate}
\end{corollary}
\begin{remark}\label{rmk: smooth}
    The centered GPs with squared exponential ($k(\bm x,\bm y)=\exp(-\|\bm x-\bm y\|^2)$), rational quadratic ($k(\bm x,\bm y)= (1+\|\bm x-\bm y\|^2)^{-a}$ for some $a >0$) and periodic ($k(x,y)=\exp(-\sin(\pi(x-y))^2)$ for $x,y\in\R$) covariance kernels are therefore sample $C^\infty(\R^d)$.
\end{remark}
\subsection{Feature Kernels}\label{sec: feature kernel}
For any map $\bm\phi\colon O\to \R^m$ is associated a feature kernel
$$k(\bm x, \bm y) = \bm\phi(\bm x)^\top \bm\phi(\bm y)$$
for $\bm x,\bm y\in O$. From \thmitemref{Theorem}{thm: holder}{thm holder: general case} we obtain the following result:
\begin{proposition}\label{prop: feature}
    For $n\in\N_0$ and $\gamma\in(0,1]$, if $\phi_i\in C^{(n+\gamma)^-}_{\loc}(O)$ for all $1\leq i \leq m$ then $f\sim \GP(0,k)$ has samples in $C^{(n+\gamma)^-}_{\loc}(O)$.
\end{proposition}
The proof can be found in \hyperref[app: proofs]{Appendix B} \ref{proof: feature}.
\begin{remark}
    As a special case, linear kernels ($k(\bm x,\bm y) =\langle \bm x,\bm y\rangle$) and polynomial kernels ($k(\bm x,\bm y) = (1+\langle \bm x,\bm y\rangle)^m$ for some $m\in\N$) are feature kernels, with the feature maps being the coordinate maps or polynomials of the coordinate maps respectively. Therefore the associated centered GPs have samples in $C^\infty(\R^d)$.
\end{remark}
\begin{remark}
    Proposition \ref{prop: feature} could alternatively be obtained by the fact that such a feature GP has sample paths of the form $f(\cdot, \omega) = \sum_{i=1}^m \xi_i(\omega)\phi_i$ for some i.i.d.~standard Gaussian random variables $\xi_i\sim \mathcal N(0,1)$.
\end{remark}
\subsection{Kernel Algebra}\label{sec: kernel algebra}
The versatility of the GP framework in applications is largely due to the fact that covariance kernels can be recursively combined using algebraic operations like pointwise sums or products which allows for constructing highly flexible and complex prior models in a systematic and interpretable fashion.
It is therefore important to understand how the sample path regularity of a GP depends on the recursive structure of the covariance kernel.

In the following the $k_i$ are always assumed to be symmetric positive definite kernels. Note that in this section we work solely with general kernels, leveraging \thmitemref{Theorem}{thm: holder}{thm holder: general case}. It would then be straightforward to similarly leverage \thmitemref{Theorem}{thm: holder}{thm holder: stationary case} and \thmitemref{Theorem}{thm: holder}{thm holder: isotropic case} for stationary and isotropic kernels respectively.
\subsubsection{Conic Combinations}
Conic combinations of kernels are pointwise linear combinations of the kernel functions with positive weights, i.e.
$$k(\bm x, \bm y) = \sum_{i = 1}^m a_i k_i(\bm x, \bm y)$$
for $\bm x,\bm y\in O$, where $a_i >0$ for all $1\le i\le m$.
Sums of kernels and positive scaling of kernels are special cases of conic combinations with $a_i = 1$ and $m = 1$, respectively.
\begin{proposition}\label{prop: conic}
    For $n\in\N_0$ and $\gamma\in (0,1]$, if the $k_i$ satisfy the condition in \thmitemref{Theorem}{thm: holder}{thm holder: general case} for all $1\leq i\leq m$ then $f\sim \GP(0,k)$ has samples in $C^{(n+\gamma)^-}_{\loc}(O)$.
\end{proposition}
This follows from the fact that for $f\sim\GP(0,k)$, we can express the sample paths as $f(\cdot,\omega)= \sum_{i=1}^m \sqrt{a_i}f_i(\cdot,\omega)$ where $f_i\sim\GP(0,k_i)$ are independent, and that $C^{(n+\gamma)^-}_{\loc}(O)$ is a vector space.
\subsubsection{Products}
Define the product kernel
$$k(\bm x, \bm y) = \prod_{i=1}^m k_i(\bm x , \bm y)$$
for $\bm x,\bm y\in O$. $k$ is also positive definite \cite[Chapter 3 Theorem 1.12]{berg_harmonic_1984}.
\begin{proposition}\label{prop: product}
    For $n\in\N_0$ and $\gamma\in (0,1]$, if the $k_i$ satisfy the condition in \thmitemref{Theorem}{thm: holder}{thm holder: general case} for all $1\leq i\leq m$ then $f\sim \GP(0,k)$ has samples in $C^{(n+\gamma)^-}_{\loc}(O)$.
\end{proposition}
The proof can be found in \hyperref[app: proofs]{Appendix B} \ref{proof: product}.
\begin{remark}
    Combining Proposition \ref{prop: product} with Proposition \ref{prop: feature} we get that if $\tilde k$ satisfies the condition in \thmitemref{Theorem}{thm: holder}{thm holder: general case} for some $n\in\N_0$ and $\gamma\in (0,1]$, and $\phi\in C^{(n+\gamma)^-}_{\loc}(O)$, then defining $k(\bm x,\bm y) = \phi(\bm x)\tilde k(\bm x,\bm y)\phi(\bm y)$ for all $\bm x,\bm y\in O$, $f\sim \GP(0,k)$ has samples in $C^{(n+\gamma)^-}_{\loc}(O)$. We could equivalently prove this by noting that we can express the sample paths as $f(\cdot,\omega) = \phi(\cdot)\tilde f(\cdot,\omega)$ where $\tilde f\sim \GP(0,\tilde k)$.
\end{remark}
\subsubsection{Tensor Products}
Let $O_i \subset \R^{d_i}$ be open sets for $1\leq i\leq m$ and $O_1\times\dots\times O_m =: O \subset \R^d = \R^{d_1}\times\dots\times \R^{d_m}$. For $k_i \colon O_i\times O_i\to\R$ for $1\leq i\leq m$, the tensor product kernel $k\colon O\times O\to \R$, is given by
$$k(\bm x,\bm y) = (k_1\otimes\dotsb\otimes k_m)(\bm x, \bm y) := \prod_{i=1}^m k_i(\bm x_i , \bm y_i )$$
for $\bm x,\bm y\in \R^d$, where $\bm x_i, \bm y_i\in \R^{d_i}$ are the orthogonal projections of $\bm x$, $\bm y$ respectively onto $\R^{d_i}$.

GPs with tensor product covariance kernels are interesting as they allow \emph{spherically asymmetric} sample path regularity. We need the following definition to formalise this.
\begin{definition}
    Let $n_1,\dots,n_m\in \N_0$ and $\gamma_1,\dots,\gamma_m\in [0,1]$.
    \begin{enumerate}[label=(\arabic*)]
        \item $C^{(n_1,\gamma_1) \otimes \dotsb \otimes (n_m,\gamma_m)}_{\loc}(O_1\times\dotsb\times O_m)$ is the space of functions $f$ on $O$ for which $\partial^{\bm\alpha}f$ exists for all multi-indices $\bm\alpha\in\N_0^d$ with $|\bm\alpha_i|\leq n_i$, where $\bm\alpha_i\in \N_0^{d_i}$ is the projection of $\bm\alpha$ onto $\N_0^{d_i}$, and such that the highest order partial derivatives satisfy a Hölder condition of the form:
            for all compact subsets $K \subset O$ there is a constant $C_K>0$ such that
            \begin{equation*}
                |\partial^{\bm\alpha}f(\bm x) -\partial^{\bm\alpha}f(\bm y)| \leq C_K\|\bm x_i - \bm y_i\|^{\gamma_i}
            \end{equation*}
            for all $1\leq i\leq d$, $\bm x,\bm y\in K$ with $\bm x_j = \bm y_j$ for $j \neq i$, and $|\bm\alpha_i| = n_i$.\\
    \item $C^{(n_1+\gamma_1)^- \otimes \dotsb \otimes (n_m+\gamma_m)^-}_{\loc}(O_1\times\dotsb\times O_m) :=\bigcap_{n_i'+\gamma_i'<n_i+\gamma_i} C^{(n_1',\gamma_1') \otimes \dotsb \otimes (n_m',\gamma_m')}_{\loc}(O_1\times\dotsb\times O_m)$.
    \end{enumerate}
\end{definition}
The following proposition generalises Theorem \ref{thm: holder}.
\begin{proposition}\label{prop: tensor product} For $n_1,\dots,n_m\in \N_0$ and $\gamma_1,\dots,\gamma_m\in (0,1]$, $f\sim \GP(0,k)$ has samples in $C^{(n_1+\gamma_1)^-\otimes\dots\otimes(n_m+\gamma_m)^-}_{\loc}(O_1 \times \dotsb \times O_m)$ if and only if the $k_i$ satisfy the condition in \thmitemref{Theorem}{thm: holder}{thm holder: general case} with $n=n_i$ and $\gamma=\gamma_i$ for all $1\leq i\leq m$.
\end{proposition}
The proof of Proposition \ref{prop: tensor product} follows the same steps as the one of Theorem \ref{thm: holder} in \hyperref[app: holder thm]{Appendix A}.
\begin{remark}
    We could generalise Theorem \ref{thm: holder} to allow for even more flexibility in the asymmetrical derivatives in Theorem \ref{thm: holder}. This can be done by defining Hölder spaces extending the spaces $C^A(O)$, where $A\subset \N_0^d$ is a downward closed set of multi-indices which specifies which partial derivatives exist \cite[Definition B.9]{pfortner_physics-informed_2023}.
\end{remark}
\subsubsection{Coordinate Transformations}\label{sec: coord}
For a map $\bm\varphi\colon O \to \R^m$ and a symmetric positive definite kernel $\tilde k$ defined on $\Im(\bm\varphi)\times\Im(\bm\varphi)$, we have a kernel
$$k(\bm x, \bm y) = \tilde k(\bm\varphi(\bm x), \bm\varphi(\bm y)).$$
\begin{proposition}\label{prop: coord}
    For $n\in\N_0$ and $\gamma,\delta\in (0,1]$, if $\tilde k$ satisfies the condition in \thmitemref{Theorem}{thm: holder}{thm holder: general case} and $\varphi_i\in C^{(n+\delta)^-}_{\loc}(O)$ for all $1\leq i\leq m$ then $f\sim \GP(0,k)$ has samples in $C^{(n+\gamma\delta)^-}_{\loc}(O)$.
\end{proposition}
This follows from the fact that for $\tilde f\sim GP(0,\tilde k)$ we have $\tilde f\circ\bm\varphi\sim\GP(0,k)$, and that the composition of a $C^{n,\zeta}_{\loc}(O)$ with a $C^{n,\epsilon}_{\loc}(O)$ function is $C^{n,\epsilon\zeta}_{\loc}(O)$.
\subsection{Manifold Kernels}\label{sec: manifold}
We can generalise the results in the present work to manifold domains; instead of considering GPs on open subsets $O\subset \R^d$, we consider GPs on manifolds $M$. There exists two ways of constructing positive definite kernels, and hence GPs, on manifolds $M$.

\emph{Extrinsic kernels} are defined by viewing $M\subset \R^d$, taking a positive definite kernel $k$ on $\R^d$, and restricting it to $M$. This restriction is analogous to a coordinate transformation (Section \ref{sec: coord}), since $k|_{M\times M}(\bm x,\bm y) = k(\bm\iota(\bm x),\bm\iota(\bm y))$ for $\bm x,\bm y\in M$, where $\bm \iota: M\to \R^d$ is the inclusion map. Therefore, if the manifold is smooth, $k|_{M\times M}$ has the same regularity as $k$.

\emph{Intrinsic kernels} are constructed directly on the manifold. Examples of these include the intrinsic Matérn and heat kernels \cite{borovitskiy_matern_2020}, and the hyperbolic secant kernel \cite{da_costa_invariant_2023}. In this case, since the sample path regularity results in this work are all local, we may study sample path regularity in each coordinate patch $O\subset M$, treating them as open subsets of $\R^d$.

\section{Discussion on Sobolev Regularity}\label{sec: sobolev}
In this section we discuss how the regularity of the kernel affects the weak differentiability of the sample paths. Specifically, we consider the important $L^2$-Sobolev regularity. The result in this section, Theorem \ref{thm: sobolev}, is a direct consequence of \cite{scheuerer_regularity_2010} combined with our Lemma \ref{lem equiv}. The more general $L^p$-Sobolev regularity is studied in \cite{henderson_sobolev_2024}. This reference includes in particular an alternative characterisation for $p=2$ \cite[Proposition 4.4]{henderson_sobolev_2024}.

To characterise only the local regularity of the sample paths, as for Hölder spaces, we define the local pre-Sobolev spaces.

\begin{definition}[Local pre-Sobolev spaces]\label{def: sobolev}
    Let $n\in\N_0$. The local pre-Sobolev space $\mathcal{H}^n_{\loc}(O)$ is the space of functions\footnote{A (local) Sobolev space is an $L^2$ quotient of a (local) pre-Sobolev space. Since we are interested in \emph{function} spaces, pre-Sobolev spaces are the correct objects of interest for us.} $f$ on $O$ for which for every compact $K\subset O$ the $L^2$ weak derivative $\partial^{\bm\alpha}_wf$ exists on $K$ for all multi-index $\bm \alpha\in\N_0^d$ with $|\bm\alpha|\leq n$, i.e.~for all such $\bm \alpha$ there is a function $\partial^{\bm\alpha}_wf\in L^2(K)$ such that
    $$\int_K\partial^{\bm\alpha}_wf(\bm x) \varphi(\bm x)\,d\bm x= (-1)^{|\bm \alpha|}\int_K f(\bm x) \partial^{\bm\alpha}\varphi(\bm x)\,d\bm x$$
    for all $\varphi\in C^\infty(O)$.
\end{definition}

\begin{theorem}[Sample path Sobolev regularity]\label{thm: sobolev}
    Let $n\in\N$.
    The process $f \sim \GP(0,k)$ has samples in $\mathcal{H}^n_{\loc}(O)$ if,
    \begin{enumerate}[label=(\arabic*)]
        \item \label{thm sobolev: general case}
            for general $k$, $k\in C^{n\otimes n}(O\times O)$.
        \item \label{thm sobolev: stationary case}
            for stationary $k(\bm x, \bm y) = k_\delta(\bm x - \bm y)$, $\partial^{\bm \alpha} k_\delta$ exists at $\bm 0$ for all $|\bm \alpha| \leq 2n$.
        \item \label{thm sobolev: isotropic case}
            for isotropic $k(\bm x, \bm y) = k_r(\| \bm x - \bm y \|)$, $k_r^{(j)}$ exists at $0$ for all $j \leq 2n$.
    \end{enumerate}
    If we assume that $k$ is continuous, or more generally that $f$ has a measurable modification, then the converse holds in the scenarios \ref{thm sobolev: stationary case} and \ref{thm sobolev: isotropic case}.
\end{theorem}
\begin{proof}
    \ref{thm sobolev: general case} follows from the general theorem for second order measurable random fields \cite[Theorem 1]{scheuerer_regularity_2010}. \ref{thm sobolev: stationary case} and \ref{thm sobolev: isotropic case} follow from its corollary \cite[Corollary 1]{scheuerer_regularity_2010}. The converse in \ref{thm sobolev: stationary case} follows from \cite[Proposition 1]{scheuerer_regularity_2010}. Finally, if \ref{thm sobolev: isotropic case} holds then, by the propagation of regularity from \cite{gneiting_derivatives_1999}, we have $k_r\in C^{2n}(\R)$. So by Lemma \ref{lem equiv} \ref{lem equiv: isotropic case} $\Rightarrow$ \ref{lem equiv: stationary case} (to be precise an adapted version without the Hölder conditions) we deduce \ref{thm sobolev: stationary case}, and hence we obtain the converse statement in this case too.
\end{proof}
As observed in \cite{scheuerer_regularity_2010}, Sobolev differentiability is a natural regularity notion for sample paths of general random fields as it can be deduced from continuous mean square differentiability. In the case of GPs however, if, in addition to the condition in \thmitemref{Theorem}{thm: sobolev}{thm sobolev: general case}, we assume Hölder control on the highest order partial derivatives of the kernel at the diagonal as in \thmitemref{Theorem}{thm: holder}{thm holder: general case}, we deduce strong differentiability of the sample paths. The Hölder condition in \thmitemref{Theorem}{thm: holder}{thm holder: general case} being a weak one, one can usually not obtain a greater number of weak derivatives using Theorem \ref{thm: sobolev} than of strong derivatives using Theorem \ref{thm: holder}. The reverse implications in \thmitemref{Theorem}{thm: sobolev}{thm sobolev: stationary case} and \thmitemref{Theorem}{thm: sobolev}{thm sobolev: isotropic case} show that, in the stationary and isotropic cases, this is not a limitation of the theorem, but an inherent property of stationary and isotropic GPs. For example, the proof of Proposition \ref{prop: matern} \hyperref[app: proofs]{Appendix B} \ref{proof: matern} shows that, by the converse to \thmitemref{Theorem}{thm: sobolev}{thm sobolev: isotropic case}, even in the edge case $\nu\in\N$, the samples of the Matérn GPs have as many Sobolev derivatives as of strong derivatives (see also \cite{henderson_sobolev_2024}).

However, in the non-stationary case, one can in some cases obtain maximally more Sobolev derivatives than strong derivatives \cite{scheuerer_regularity_2010}, \cite{henderson_sobolev_2024}. One way to construct non-stationary GPs with more weak derivatives than strong derivatives is through feature kernels, with features that admit more weak derivatives than strong derivatives (see Section \ref{sec: feature kernel} and \cite[Example 3.3]{henderson_sobolev_2024}). This also shows that it is not possible to obtain a converse for \thmitemref{Theorem}{thm: sobolev}{thm sobolev: general case} as we have for \thmitemref{Theorem}{thm: sobolev}{thm sobolev: stationary case} and \thmitemref{Theorem}{thm: sobolev}{thm sobolev: isotropic case}.

Also note that deducing sample path Hölder regularity from applying Theorem \ref{thm: sobolev} and Sobolev embedding theorems introduces a superfluous dependence on dimension, and yields weaker results than the dimension independent Theorem \ref{thm: holder}.

We therefore expect Theorem \ref{thm: sobolev} to have less practical value for GPs as its general version \cite[Theorem 1]{scheuerer_regularity_2010} for second order measurable random fields. For GPs we encourage the use of Theorem \ref{thm: holder}.

An interesting avenue of future research would be to generalise Theorem \ref{thm: sobolev} to fractional Sobolev spaces. One may wonder whether the condition $k\in C^{n\otimes n}(O\times O)$ in addition to Hölder condition on the highest order partial derivatives $|\partial^{\bm \alpha,\bm \beta} k(\bm x+\bm h,\bm x+\bm h)-\partial^{\bm \alpha,\bm \beta}k(\bm x+\bm h,\bm x)-\partial^{\bm \alpha,\bm\beta}k(\bm x,\bm x+\bm h)+\partial^{\bm \alpha,\bm \beta}k(\bm x,\bm x)| = \mathcal{O}(\|\bm h\|^{2\gamma})$ as $\bm h \to \bm 0$, locally uniformly in $\bm x\in O$, for all $|\bm \alpha|=|\bm\beta|=n$ and some $\gamma\in (0,1)$, is enough to guarantee that the GP has samples in $\mathcal H^{n+\gamma}_{\loc}(O)$ (as opposed to $C^{(n+\gamma)^-}_{\loc}(O)$, which is what one would obtain from Theorem \ref{thm: holder}).

\section{Conclusion}

In this paper we have investigated the Hölder and Sobolev sample path regularity of Gaussian processes in terms of the covariance kernel, with a combination of results from the literature and novel ones. We focused on providing results that are readily applicable, and demonstrated this with the most commonly used examples in practice. We hope that this work will serve as a comprehensive guide on the regularity of Gaussian process sample paths for practitioners.

\begin{figure}
\centering
\includegraphics[height=174pt]{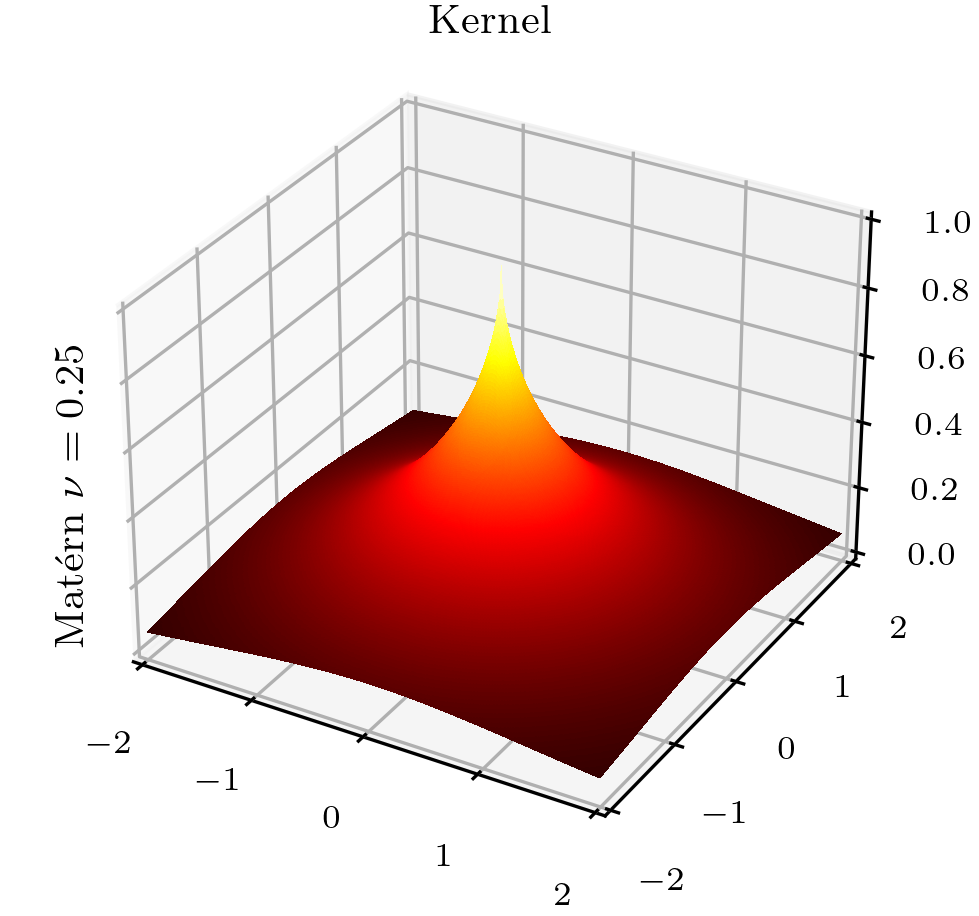}
\includegraphics[height=174pt]{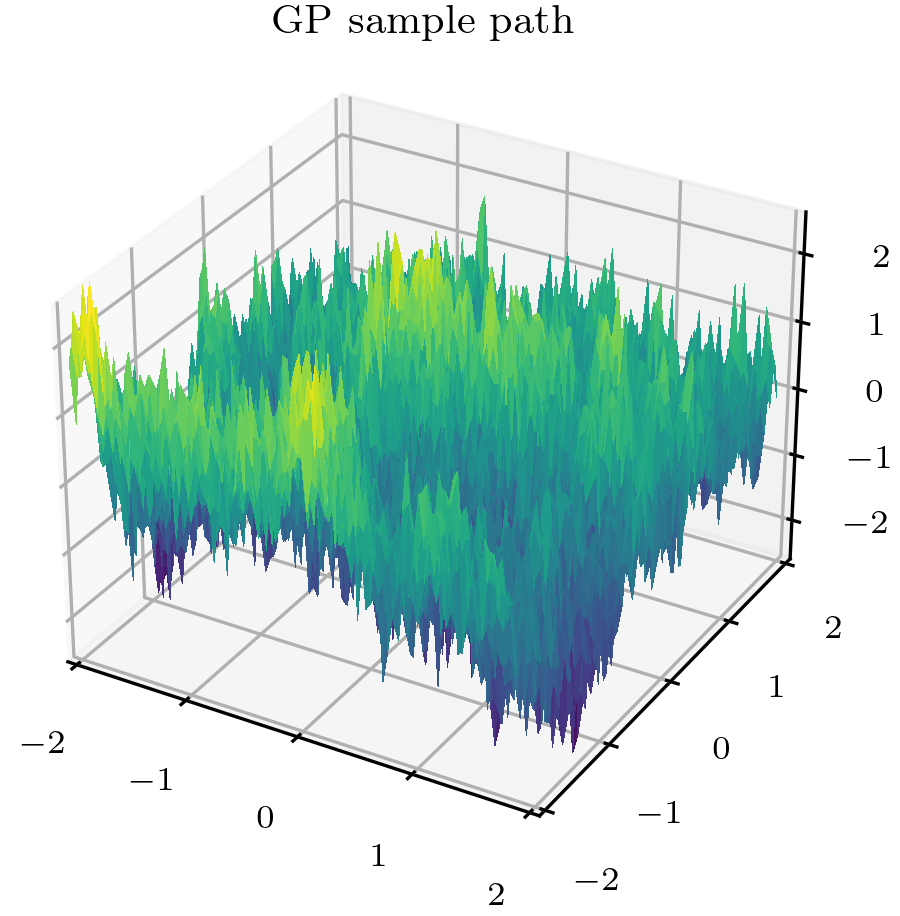} \\
\includegraphics[height=170pt]{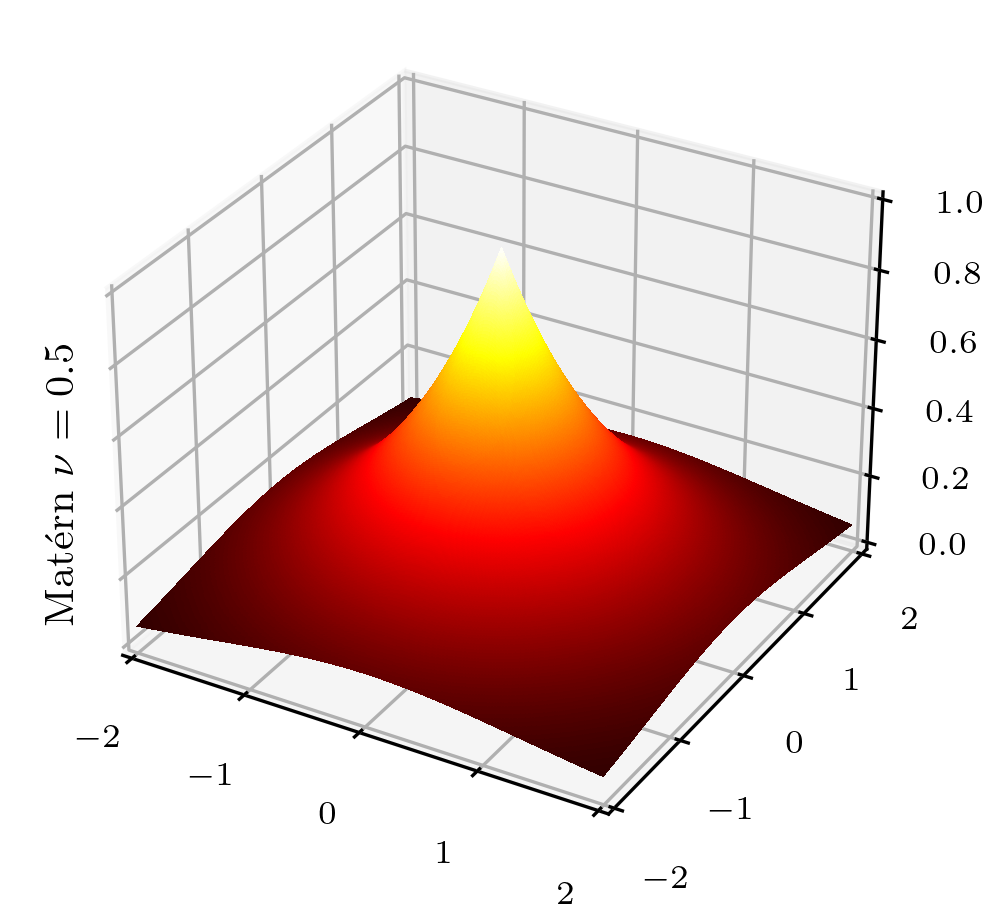}
\includegraphics[height=170pt]{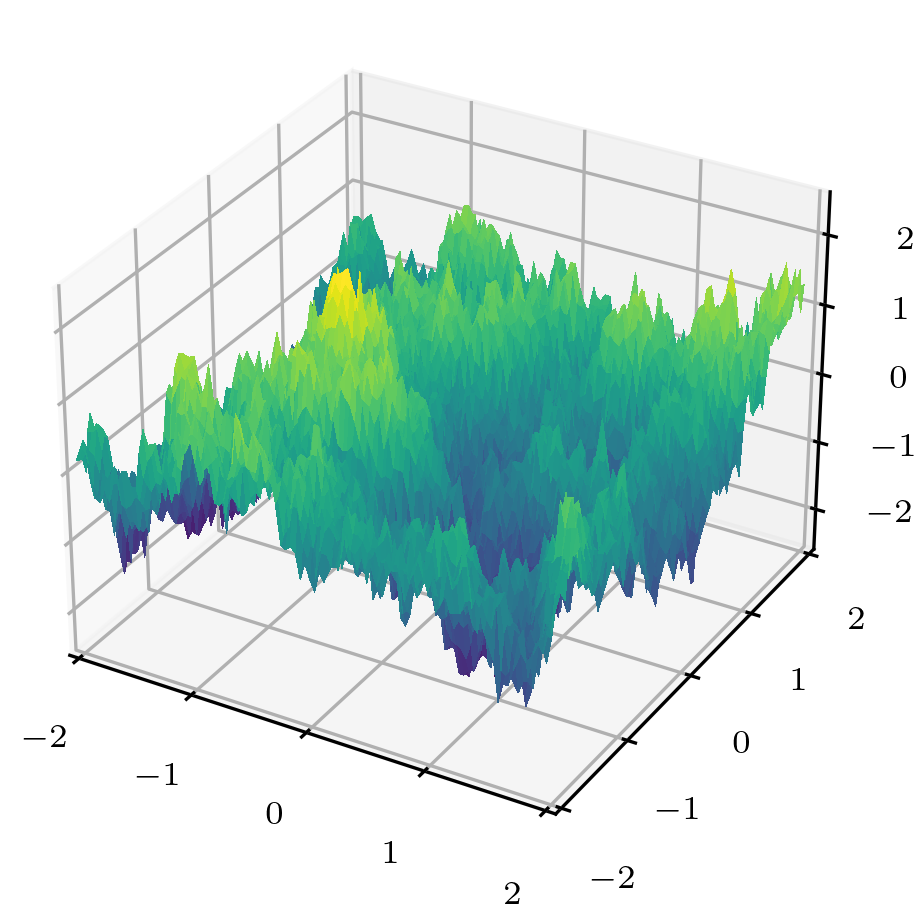} \\
\includegraphics[height=170pt]{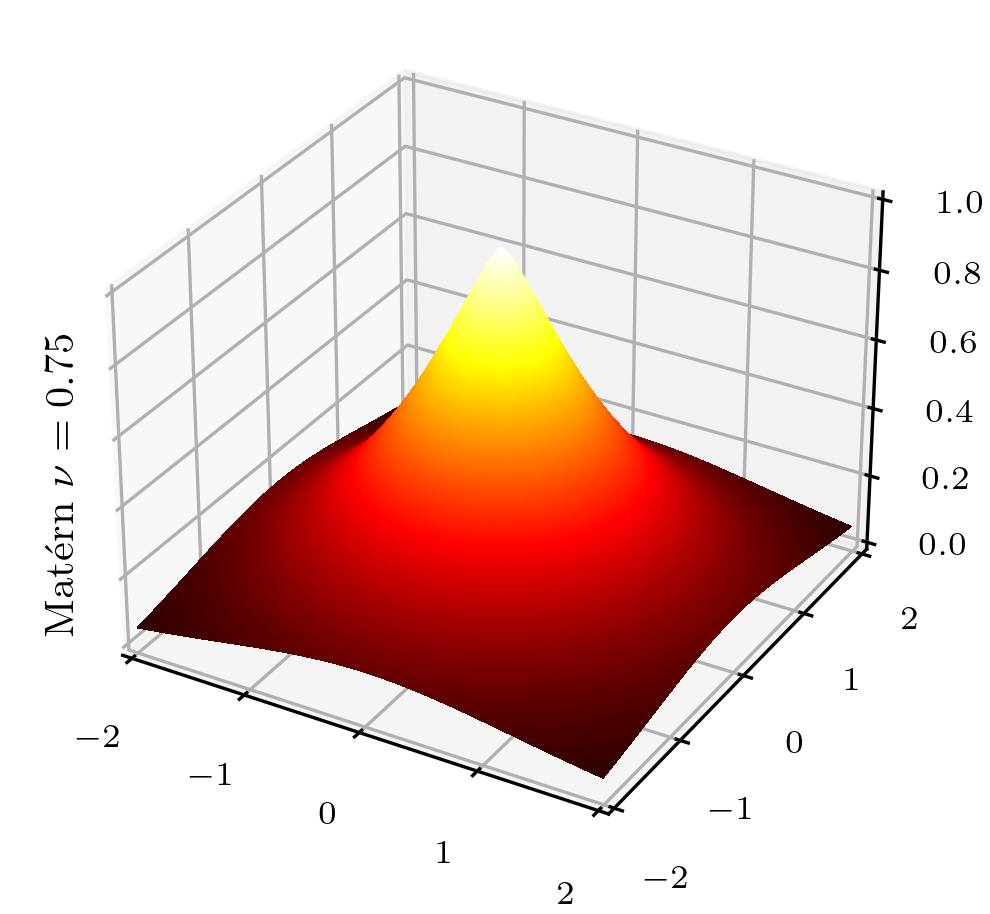}
\includegraphics[height=170pt]{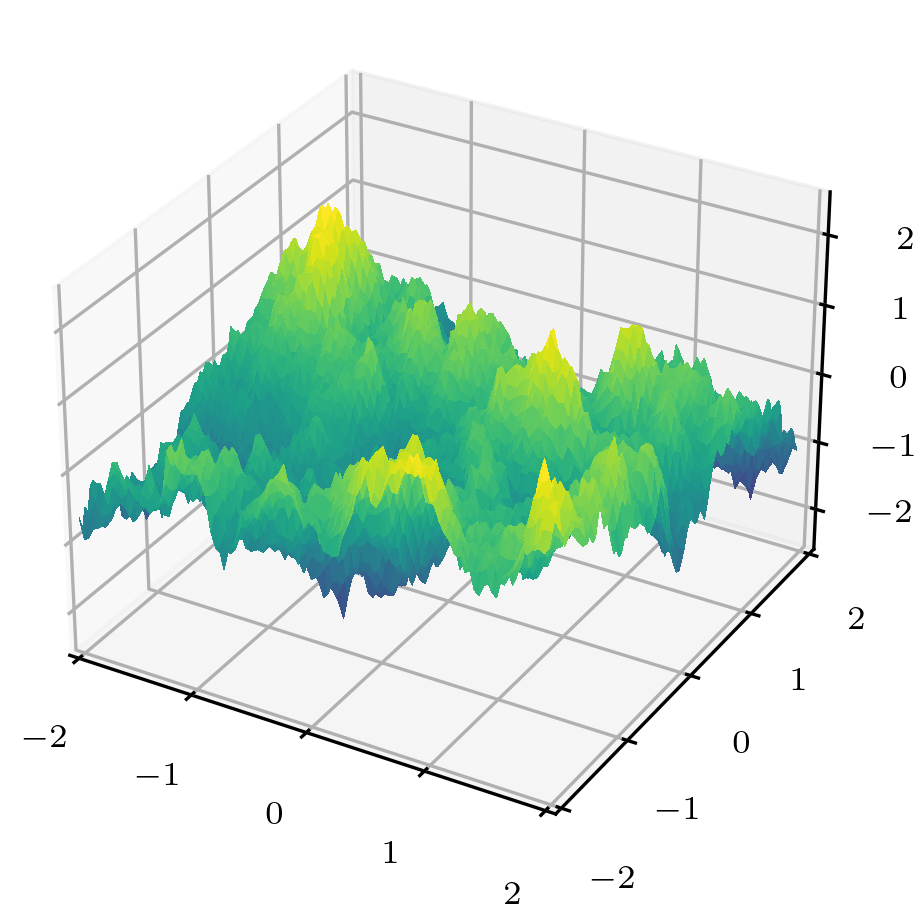}
\end{figure}
\begin{figure}
\centering
\includegraphics[height=174pt]{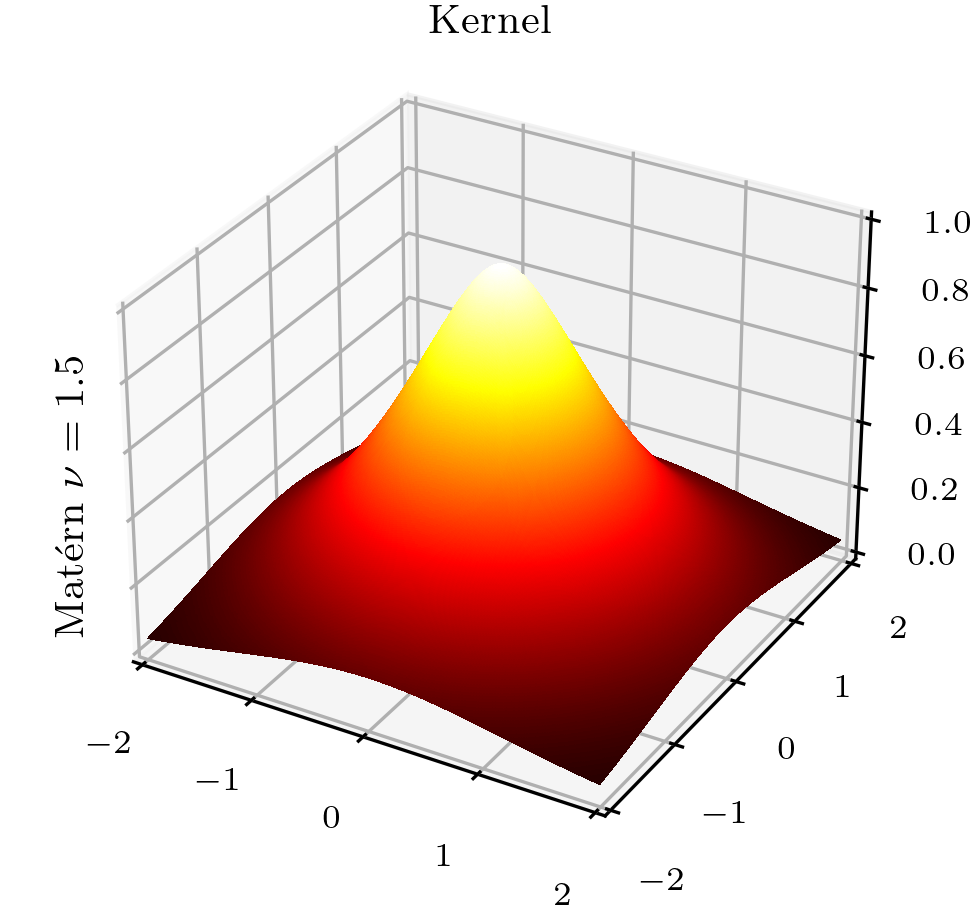}
\includegraphics[height=174pt]{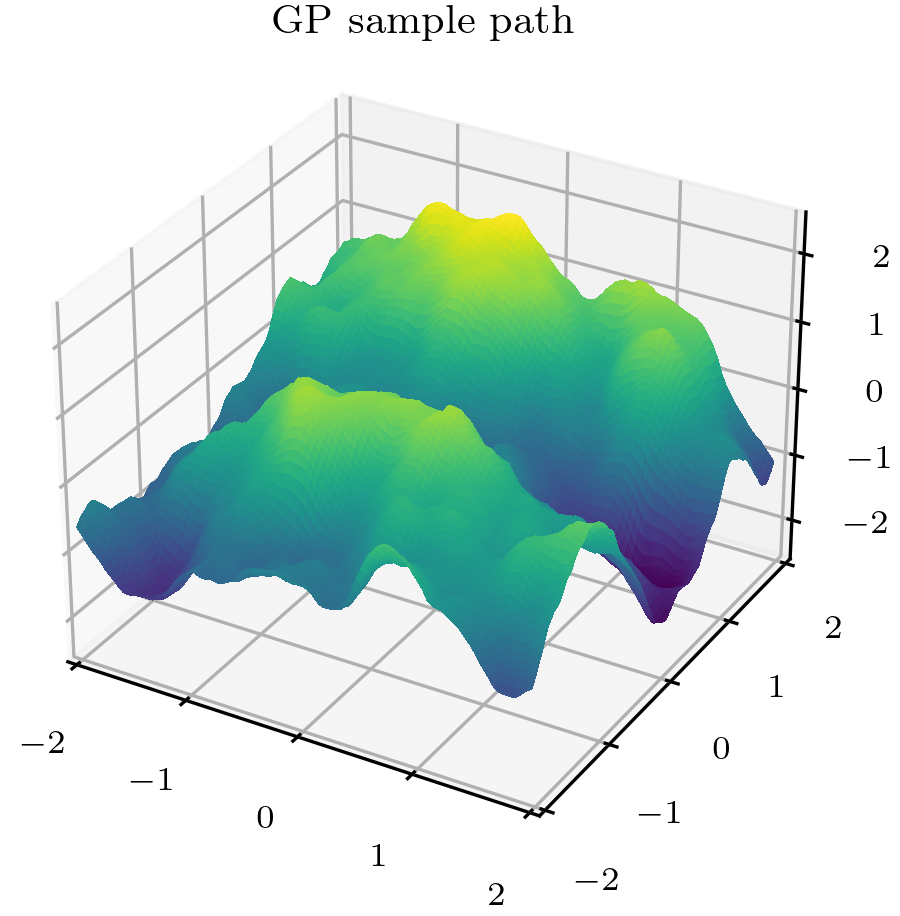} \\
\includegraphics[height=170pt]{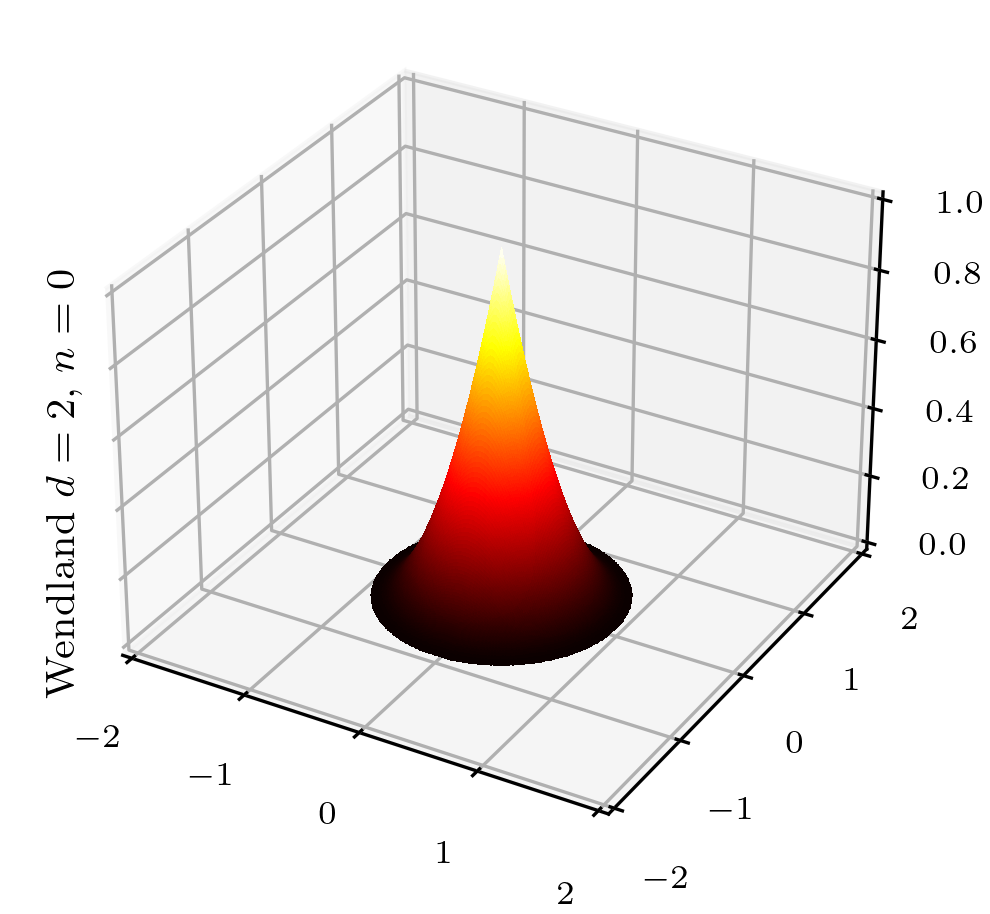}
\includegraphics[height=170pt]{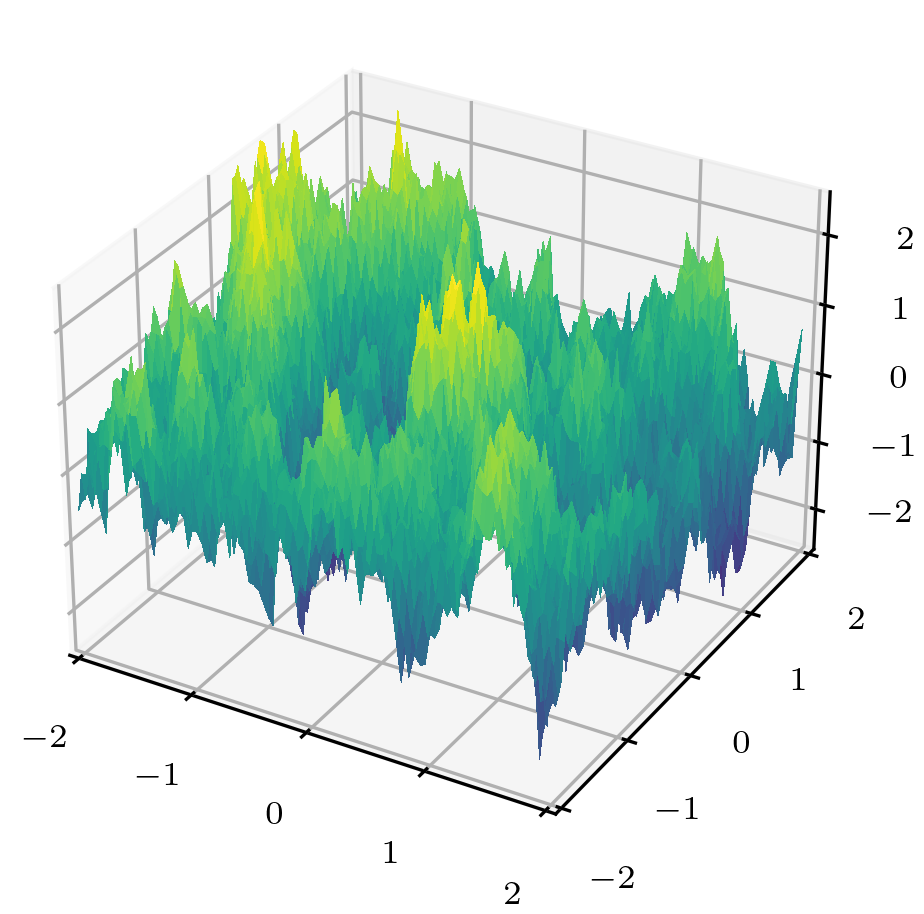} \\
\includegraphics[height=170pt]{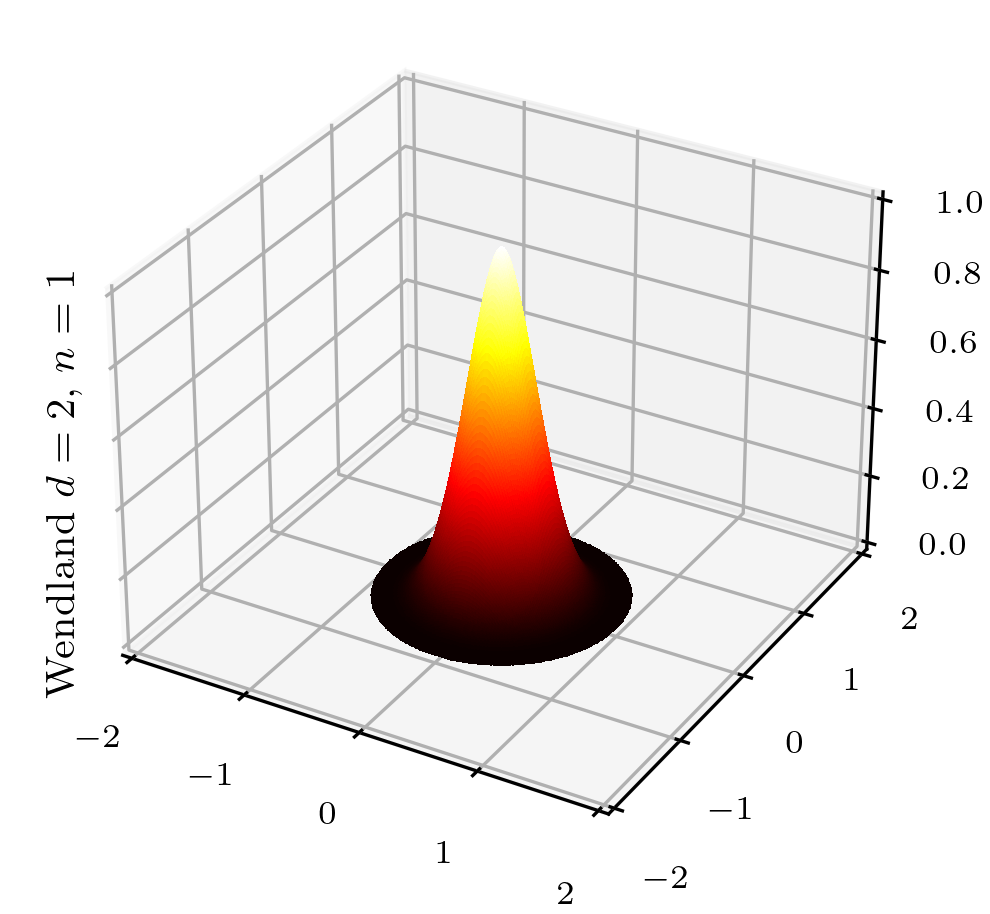}
\includegraphics[height=170pt]{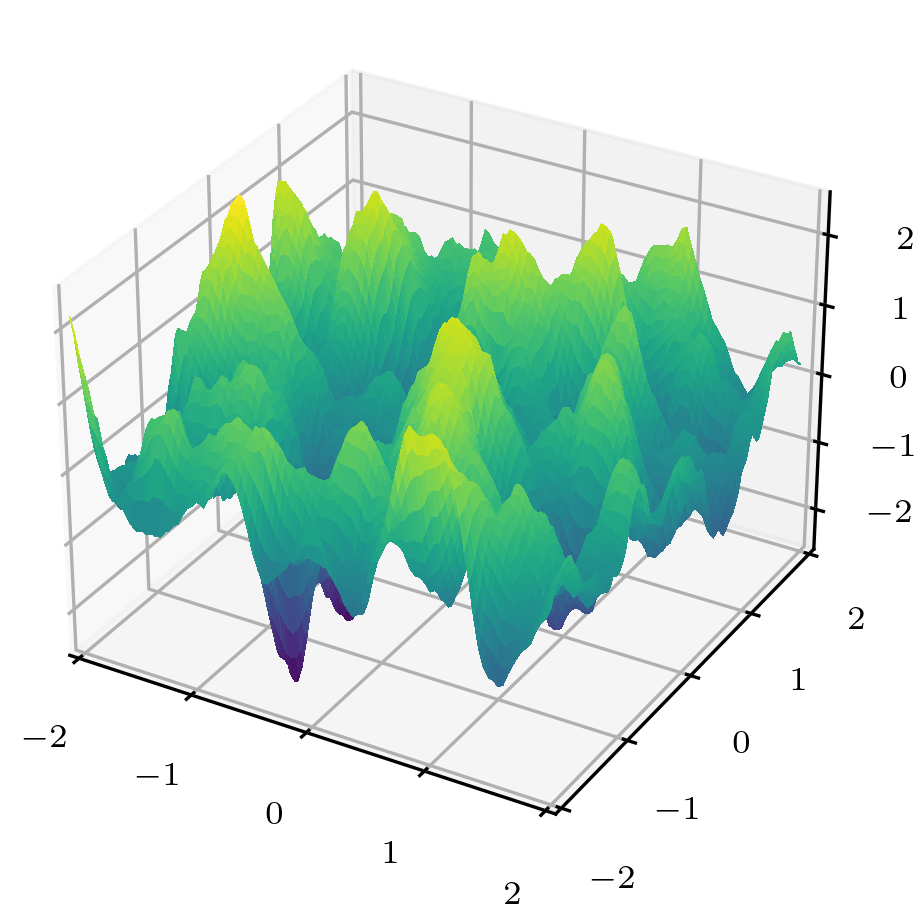} \\
\end{figure}
\begin{figure}
\centering
\includegraphics[height=174pt]{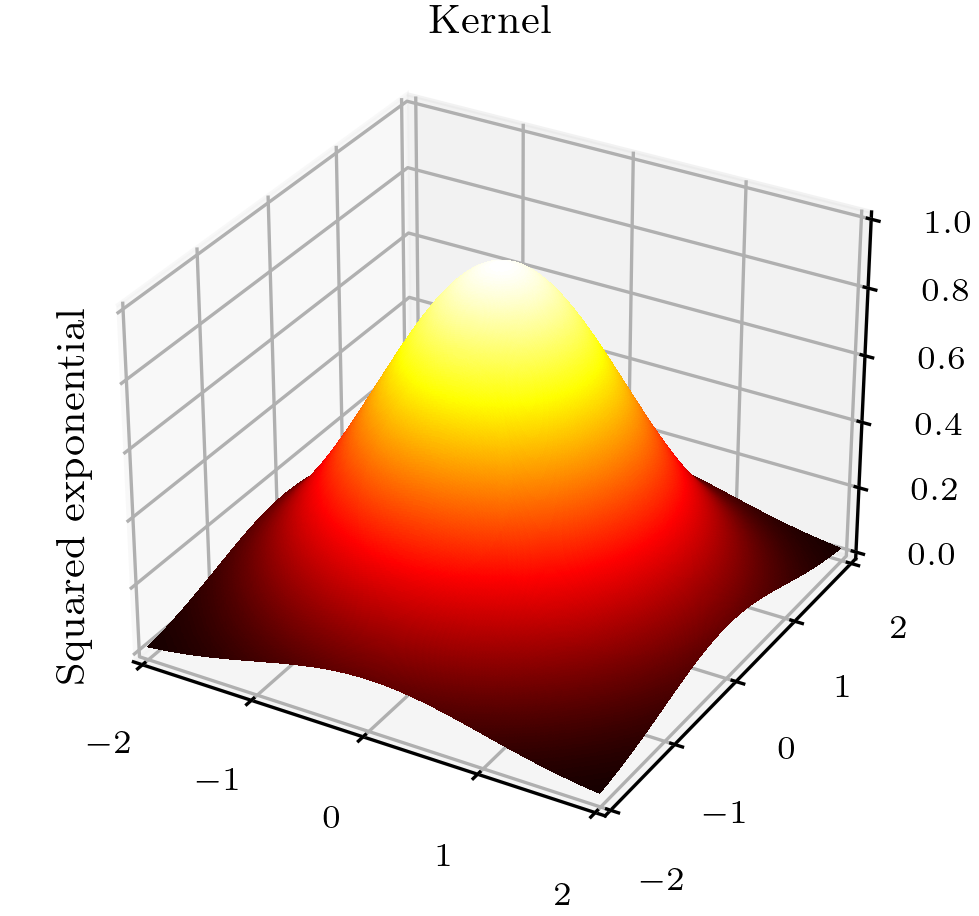}
\includegraphics[height=174pt]{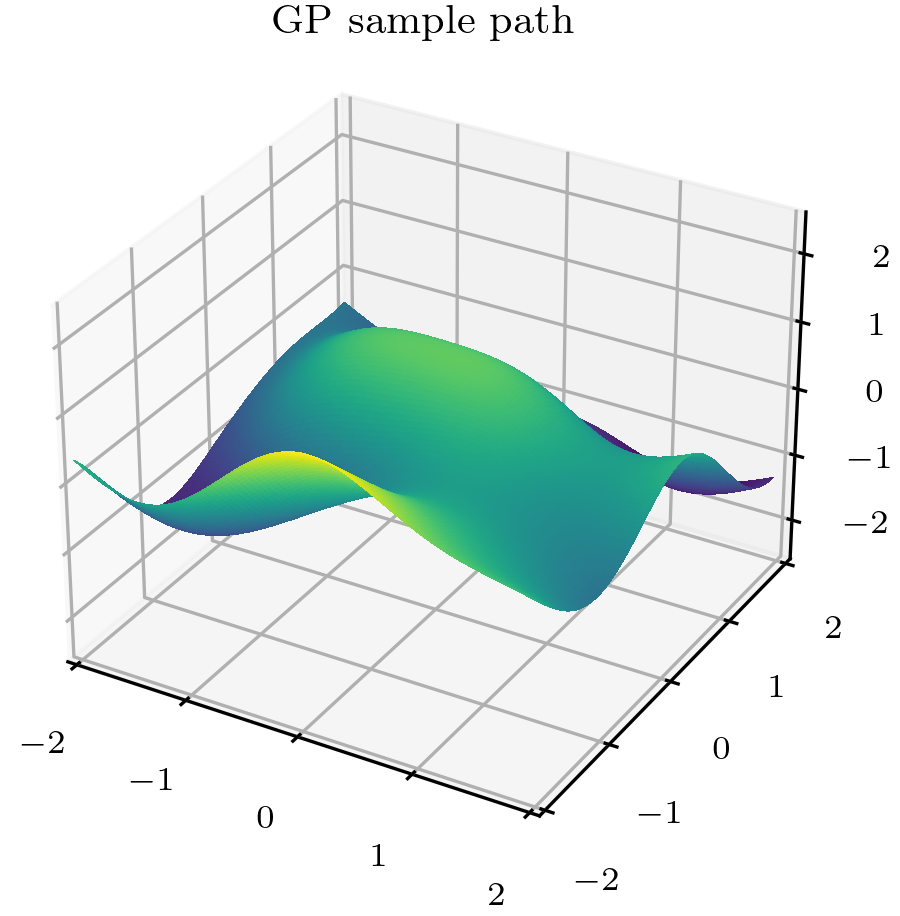} \\
\includegraphics[height=170pt]{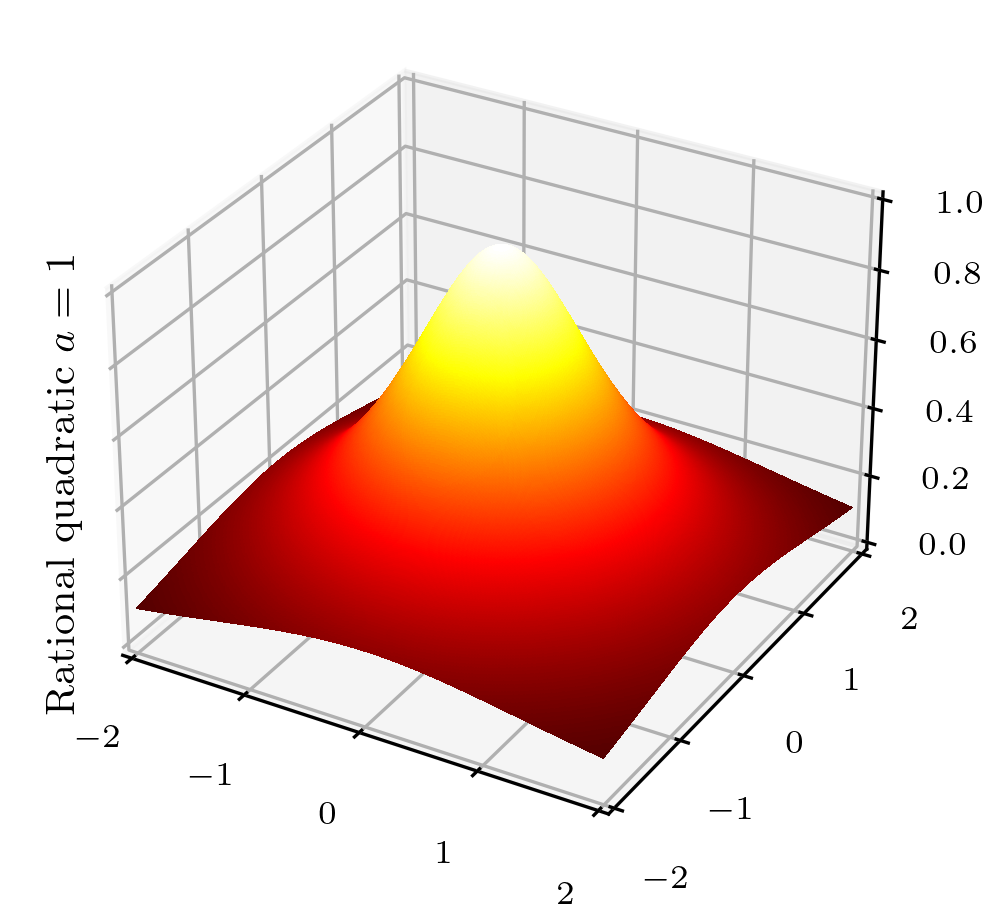}
\includegraphics[height=170pt]{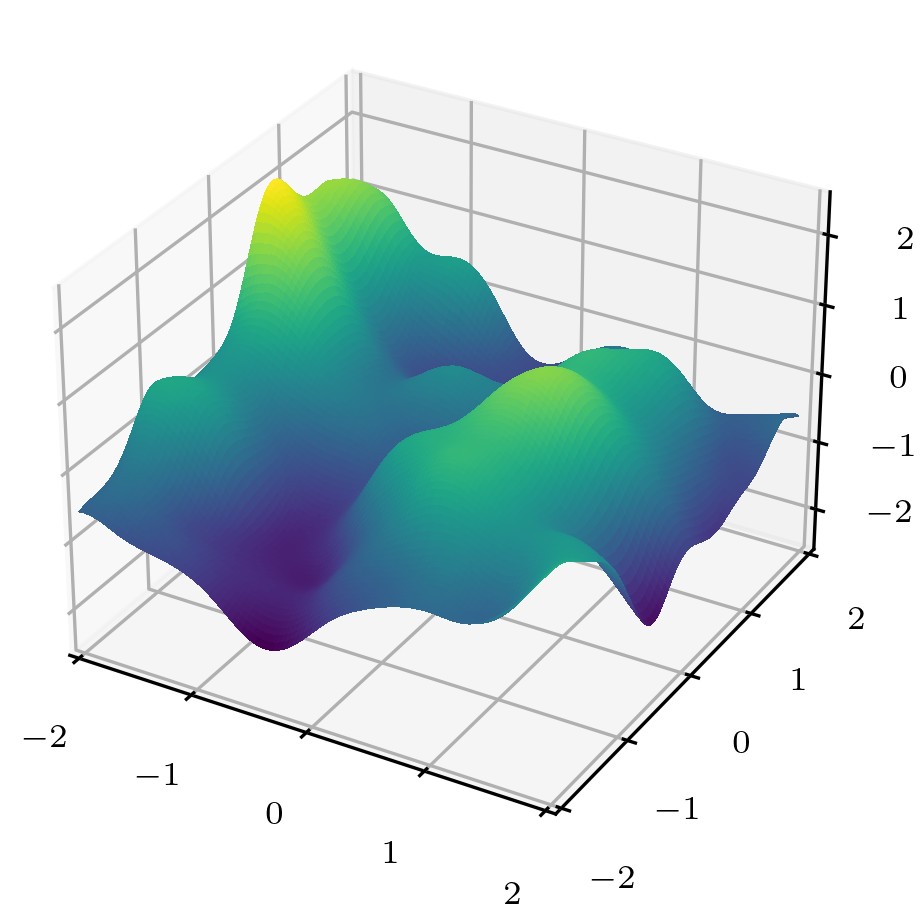} \\
\includegraphics[height=170pt]{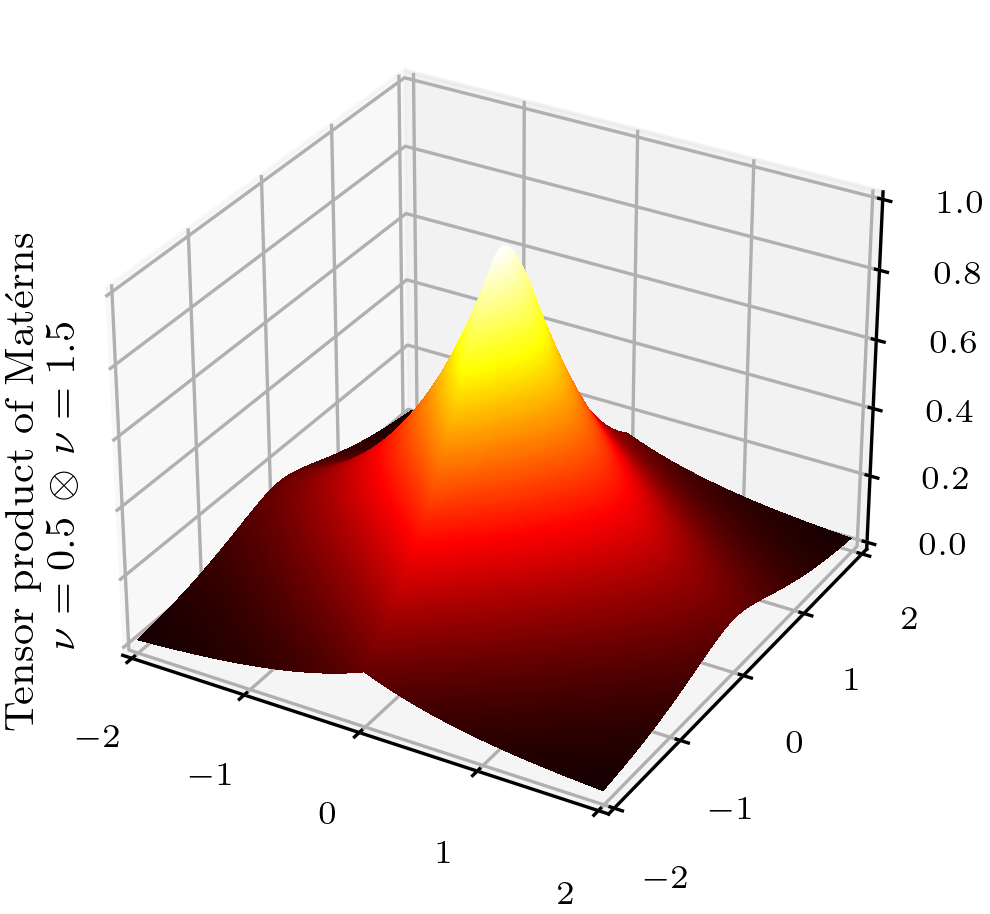}
\includegraphics[height=170pt]{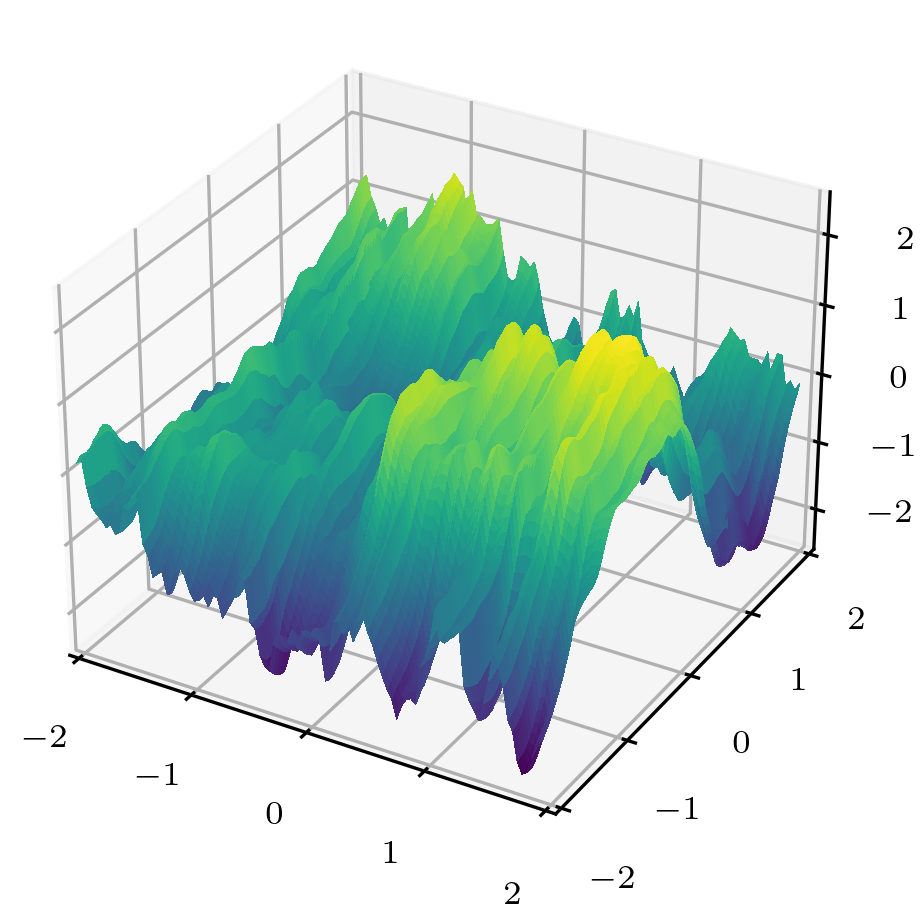}
\end{figure}

\section*{Acknowledgments}
    NDC, MP and PH gratefully acknowledge financial support by
    the European Research Council through ERC StG Action 757275 / PANAMA;
    the DFG Cluster of Excellence “Machine Learning - New Perspectives for Science”, EXC 2064/1, project number 390727645;
    the German Federal Ministry of Education and Research (BMBF) through the Tübingen AI Center (FKZ: 01IS18039A);
    and funds from the Ministry of Science, Research and Arts of the State of Baden-Württemberg.
    NDC acknowledges the support of the Fonds National de la Recherche, Luxembourg.
    The authors thank the International Max Planck Research School for Intelligent Systems (IMPRS-IS) for supporting NDC and MP.

%% file: appendix.tex
\section{Proof of Theorem \ref{thm: holder}}
\label{app: holder thm}
We start by showing the equivalences of the various kernel conditions in Theorem \ref{thm: holder}, namely \ref{thm holder: general case} $\Leftrightarrow$ \ref{thm holder: stationary case} in the stationary case, and \ref{thm holder: general case} $\Leftrightarrow$ \ref{thm holder: stationary case} $\Leftrightarrow$ \ref{thm holder: isotropic case} in the isotropic case.
\begin{lemma}
    \label{lem equiv}
    Let $n \in \N_0$ and $\epsilon \in (0,1]$.
    Then the condition
    \begin{enumerate}[label=(\arabic*)]
        \item \label{lem equiv: general case}%
            \begin{itemize}
                \item $k\in C^{n\otimes n}(O\times O)$,
                \item $|\partial^{\bm \alpha,\bm \beta} k(\bm x+\bm h,\bm x+\bm h)-\partial^{\bm \alpha,\bm \beta}k(\bm x+\bm h,\bm x)-\partial^{\bm \alpha,\bm\beta}k(\bm x,\bm x+\bm h)+\partial^{\bm \alpha,\bm \beta}k(\bm x,\bm x)| = \mathcal{O}(\|\bm h\|^{2\epsilon})$ \\
                as $\bm h \to \bm 0$, locally uniformly in $\bm x\in O$, for all $|\bm \alpha|=|\bm\beta|=n$.
            \end{itemize}
    \end{enumerate}
    is equivalent to the following:
    \begin{enumerate}[label=(\arabic*),resume]
        \item \label{lem equiv: stationary case}%
            for stationary $k(\bm x, \bm y) = k_\delta(\bm x -\bm y)$,
            \begin{itemize}
                \item $k_\delta\in C^{2n}(\R^d)$,
                \item $|\partial^{\bm\alpha}k_\delta(\bm h)-\partial^{\bm\alpha}k_\delta(\bm 0)| = \mathcal{O}(\|\bm h\|^{2\epsilon})$ as $\bm h \to \bm 0$ for all $|\bm \alpha|=2n$.
            \end{itemize}
        \item \label{lem equiv: isotropic case}%
            for isotropic $k(\bm x, \bm y) = k_r(\| \bm x - \bm y \|)$,
            \begin{itemize}
                \item $k_r \in C^{2n}(\R)$,
                \item $|k_r^{(2n)}(h)-k_r^{(2n)}(0)| = \mathcal{O}(|h|^{2\epsilon})$ as $h \to 0$.
            \end{itemize}
    \end{enumerate}
\end{lemma}
\begin{proof}\textbf{of Lemma \ref{lem equiv} \ref{lem equiv: general case} $\Leftrightarrow$ \ref{lem equiv: stationary case}}
    Let $k$ be stationary, i.e.~$k(\bm x, \bm y) = k_\delta(\bm x - \bm y)$.

    \textbf{\ref{lem equiv: stationary case} $\Rightarrow$ \ref{lem equiv: general case}:} For $|\bm \alpha|, |\bm\beta|\leq n$,
    $$\partial^{\bm\alpha,\bm\beta}k(\bm x, \bm y) = (-1)^{|\bm\beta|}\partial^{\bm\alpha+\bm\beta}k_\delta (\bm x-\bm y)$$
    for all $\bm x,\bm y \in \R^d$, by the chain rule. So the existence and the continuity of $\partial^{\bm \alpha+\bm\beta}k_\delta$ implies the existence and the continuity of $\partial^{\bm \alpha,\bm\beta}k$. Moreover when $|\bm\alpha|=|\bm\beta|=n$,
    \begin{equation}\label{eq: identity diag vs sing}
    \begin{aligned}
    &|\partial^{\bm \alpha,\bm \beta}k(\bm x+\bm h,\bm x+\bm h)-\partial^{\bm \alpha,\bm \beta}k(\bm x+\bm h,\bm x)-\partial^{\bm \alpha,\bm \beta}k(\bm x,\bm x+\bm h) +\partial^{\bm \alpha,\bm \beta}k (\bm x,\bm x)|\\
    &= 2|\partial^{\bm \alpha+\bm\beta}k_\delta (\bm 0)-\partial^{\bm \alpha+\bm\beta}k_\delta (\bm h)|
    \end{aligned}
    \end{equation}
    for all $\bm x, \bm h\in \R^d$, so the Hölder conditions on the highest order partial derivatives of $k$ at the diagonal and of $k_\delta$ at $\bm 0$ correspond.

    \textbf{\ref{lem equiv: general case} $\Rightarrow$ \ref{lem equiv: stationary case}:} For $|\bm\alpha|\leq 2n$, say $\bm\alpha = \bm\beta+\bm\gamma$ with $|\bm\beta|,|\bm\gamma|\leq n$, $\partial^{\bm \alpha}k_\delta(\bm x)$ exists and is equal to $(-1)^{|\bm \gamma|}\partial^{\bm \beta,\bm \gamma}k(\bm x,\bm 0)$, for all $\bm x\in \R^d$. Moreover the Hölder conditions on the partial derivatives of $k$ and $k_\delta$ correspond by Equation (\ref{eq: identity diag vs sing}).
\end{proof}
To prove the second equivalence in Lemma \ref{lem equiv}, we need the following result:
\begin{lemma}
    \label{lem norm concat}
    Let $m \in \N_0$, $g \in C^m(\R\setminus\{0\})$, and $\nu \in \R$ such that $g^{(j)}(h) = \mathcal O(|h|^{\nu-j})$ as $h \to 0$, for all $0 \leq j \leq m$.
    Define $f := g\circ \|\cdot\|\colon\R^d\setminus\{\bm 0\}\to\R$.
    Then $f\in C^m(\R^d\setminus\{\bm 0\})$ and $\partial^{\bm \alpha} f(\bm h) = \mathcal O(\|\bm h\|^{\nu-j})$ as $\bm h\to \bm 0$, for $0\leq j\leq m$ and $|\bm \alpha| = j$.
\end{lemma}
\begin{proof}
    By induction on $m$.
    $m=0$ is clear; the growth/decay of $f$ at $\bm 0$ is the same as that of $g$ at $0$.
    For $m>0$, $\bm x \in \R^d\setminus\{\bm 0\}$, and $1\leq i\leq d$, note that
    \begin{equation}\label{eq: partial f}
    \partial^{\bm e_i}f(\bm x) = g'(\|\bm x\|)\frac{x_i}{\|\bm x\|}= \tilde{g}(\|\bm x\|)x_i
    \end{equation}
    where $\bm e_i$ is the $i$\textsuperscript{th} unit vector and $\tilde{g}(x) :=  \frac{g'(x)}{x}$ for $x\in \R\setminus \{0\}$.
    Clearly $\tilde{g}\in C^{m-1}(\R\setminus\{0\})$ and
    $$\tilde{g}^{(j)}(h) = \sum_{l=0}^{j} \binom{j}{l} g^{(l+1)}(h)(-1)^{j-l}(j-l)!h^{-1-j+l}=\mathcal O(|h|^{(\nu - 2) - j})$$
    as $h\to 0$, for all $0\leq j\leq m-1$.
    Let $\tilde{f} = \tilde{g} \circ \|\cdot\|$.
    By the induction hypothesis, $\partial^{\bm \alpha} \tilde{f}(\bm h) = \mathcal O(\|\bm h\|^{(\nu - 2) - j})$ as $\bm h\to \bm 0$, for $|\bm \alpha| = j$ and $0 \leq j \leq m - 1$. Thus, by (\ref{eq: partial f}),
    \begin{equation*}
        \partial^{\bm \alpha} \partial^{\bm e_i} f(\bm h)
        = \left\{ \begin{aligned}
            & h_i \partial^{\bm \alpha} \tilde{f}(\bm h) &\ & \text{if } \alpha_i = 0 \\
            & h_i \partial^{\bm \alpha} \tilde{f}(\bm h) + \partial^{\bm \alpha - \bm e_i} \tilde{f}(\bm h) && \text{if } \alpha_i > 0
        \end{aligned} \right\}
        = \mathcal O(\|\bm h\|^{\nu - (j + 1)})
    \end{equation*}
    as $\bm h\to \bm 0$.
    Consequently, $\partial^{\bm \alpha} f(\bm h) = \mathcal O(\|\bm h\|^{\nu - j})$ as $\bm h\to \bm 0$ for $|\bm \alpha| = j$ and $0 \leq j \leq m$.
    This concludes the induction step.
\end{proof}
\begin{proof}\textbf{of Lemma \ref{lem equiv} \ref{lem equiv: general case} $\Leftrightarrow$ \ref{lem equiv: isotropic case}}
    Let $k$ be isotropic, i.e.~$k(\bm x, \bm y) = k_r(\| \bm x - \bm y \|)$.
    By Lemma \ref{lem equiv} \ref{lem equiv: general case} $\Leftrightarrow$ \ref{lem equiv: stationary case}, it suffices to show that the condition on $k_\delta$ in \ref{lem equiv: stationary case} is equivalent to the condition on $k_r$ in \ref{lem equiv: isotropic case}.

    \textbf{\ref{lem equiv: stationary case} $\Rightarrow$ \ref{lem equiv: isotropic case}:} This follows from the fact that $x \mapsto k_\delta(x \bm e_1)$ is exactly $k_r$, and analogously $x \mapsto \partial^{j \bm e_1}k_\delta(x \bm e_1)$ is equal to $k_r^{(j)}$ for $0\leq j\leq 2n$.

    \textbf{\ref{lem equiv: isotropic case} $\Rightarrow$ \ref{lem equiv: stationary case}:}
    $k_r\colon\R \to \R$ is even, so the odd derivatives of $k_r$ which exist at 0 vanish there.
    Let
    $$g(x) := k_r(x)-\sum_{j=0}^n\frac{x^{2j}}{j!}k_r^{(2j)}(0),$$
    for all $x\in \R$.
    The Lagrange form of the remainder in an order $2n-j -1$ Taylor expansion of $g^{(j)}$ then reveals that
    $$
    g^{(j)}(h) = \frac{h^{2n-j}}{(2n-j)!} (k_r^{(2n)}(\xi_j(h))-k_r^{(2n)}(0)) = \mathcal O(|h|^{2n-j+2\epsilon})
    $$
    as $h\to 0$, for all $0\leq j\leq 2n$ and where $\xi_j(h)\in (0,h)$ (or $(h,0)$ if $h<0$). So $g \vert_{\R \setminus \{0\}}$ satisfies the conditions of Lemma \ref{lem norm concat} with $m = 2n$ and $\nu = 2n+2\epsilon$.
    Let $f:= g\circ \|\cdot\|\colon\R^d \to \R$.
    Then, by Lemma \ref{lem norm concat}, $f \vert_{\R^d \setminus \{\bm 0\}}\in C^{2n}(\R^d\setminus\{\bm 0\})$, and $\partial^{\bm \alpha} f(\bm h) = \mathcal O(\|\bm h\|^{2n+2\epsilon-j})$ as $\bm h\to \bm 0$ for $|\bm \alpha| = j$ and $0\leq j\leq 2n$.
    In particular, $\partial^{\bm \alpha} f(\bm h) \to 0$ as $\bm h \to \bm 0$, and by the mean value theorem this is sufficient to deduce that $\partial^{\bm \alpha} f(\bm 0)$ exists and is 0. Now by noting that $\|\bm x\|^{2j} =(x_1^2+\dots+x_d^2)^j$ is smooth in $\bm x$ for all $j$, we have that
    $$k_\delta(\bm x) = f(\bm x) +\sum_{j=0}^n\frac{\|\bm x\|^{2j}}{j!}k_r^{(2j)}(0)$$
    is in $C^{2n}(\R)$.
    Moreover, for $\bm h\in \R^d$ and $|\bm\alpha| = 2n$ there is a constant $C_{\bm\alpha}\in \R$ such that $\partial^{\bm \alpha}k_\delta(\bm h) = \partial^{\bm \alpha} f(\bm h)+C_{\bm\alpha}$, implying that
    $$|\partial^{\bm\alpha}k_\delta(\bm h)-\partial^{\bm\alpha}k_\delta(\bm 0)| = |\partial^{\bm \alpha} f(\bm h)-\underbrace{\partial^{\bm \alpha} f(\bm 0)}_{= 0}| = \mathcal O(\|\bm h\|^{2n+2\epsilon-2n}) = \mathcal O(\|\bm h\|^{2\epsilon})$$
    as $\bm h\to \bm 0$. Thus $k_\delta$ satisfies the Hölder condition in \ref{lem equiv: stationary case}.
\end{proof}
Note that Lemma \ref{lem equiv} is purely a result about stationary and isotropic functions; the proof does make use of the positive definiteness of $k$.
In fact, the positive definiteness of $k$ gives us additional information about its regularity, which can be seen through Lemma \ref{lem extrapol} below.

For a GP $f$, the mean-square partial derivative $\partial^{\bm e_i}_{ms}f$, when it exists, is a stochastic process defined on the same probability space such that $\partial^{\bm e_i}_{ms}f(\bm x)$ is the $L^2$ limit of $\frac{f(\bm x+h\bm e_i)-f(\bm x)}{h}$ as $h\to 0$. The existence of a mean-square partial derivative is neither sufficient nor necessary for the existence of sample derivatives. However it is sufficient if we know in addition that this derivative is mean-square continuous and has continuous samples \cite[Theorem 3.2]{potthoff_sample_2010}. Moreover, unlike sample derivatives, mean-square derivatives are directly related to the derivatives of the covariance kernel:
\begin{lemma}\label{lem extrapol}
Let $n\in\N_0$. For $f\sim \GP(0,k)$, the following conditions are equivalent:
\begin{enumerate}[label=(\arabic*)] 
    \item\label{lem extrapol: all} $k\in C^{n\otimes n}(O\times O)$, \\
    \item\label{lem extrapol: diag} $\partial^{\bm \alpha,\bm \alpha}k$ exists in a neighbourhood the diagonal and is continuous there for all $|\bm \alpha|\leq n$, \\
    \item\label{lem extrapol: ms dve} the mean-square partial derivative $\partial^{\bm\alpha}_{ms}f$ exists and is mean-square continuous for all $|\bm \alpha|\leq n$.
\end{enumerate}
Moreover, if either condition holds then we have $\partial^{\bm\alpha}_{ms}f\sim \GP(0,\partial^{\bm\alpha,\bm\alpha}k)$ for all $|\bm\alpha|\leq n$.
\end{lemma}
\begin{proof} We prove this for $n=1$; for general $n$ it then suffices to apply the same argument inductively on the partial derivatives.

\textbf{\ref{lem extrapol: all} $\Rightarrow$ \ref{lem extrapol: diag}:} This implication is clear.

\textbf{\ref{lem extrapol: diag} $\Rightarrow$ \ref{lem extrapol: ms dve}:}
Let $1\leq i\leq d$. Note that
\begin{equation}\label{eq: ms to kernel}
\begin{aligned}
    &\E\left[\left(\frac{f(\bm x+h\bm e_i)-f(\bm x)}{h} \right)\left(\frac{f(\bm y+h'\bm e_i)-f(\bm y)}{h'}\right)\right] \\
    &= \frac{\E[f(\bm x+h\bm e_i)f(\bm y +h'\bm e_i)]-\E[f(\bm x+h\bm e_i)f(\bm y)]-\E[f(\bm x)f(\bm y+h'\bm e_i)]+\E[f(\bm x)f(\bm y)]}{hh'} \\
    &= \frac{k(\bm x+h\bm e_i,\bm y +h'\bm e_i)-k(\bm x+h\bm e_i,\bm y)-k(\bm x,\bm y+h'\bm e_i) + k(\bm x,\bm y)}{hh'} \\
    &\to \partial^{\bm e_i,\bm e_i}k(\bm x,\bm y)
\end{aligned}
\end{equation}
as $h, h' \to 0$, for $\bm x,\bm y\in O$ close enough, since $\partial^{\bm e_i,\bm e_i}k$ is assumed to be continuous in a neighbourhood of the diagonal.
Thus we have
\begin{equation}\label{eq: ms convergence}
\begin{aligned}
    &\E \left[\left( \frac{f(\bm x+h\bm e_i)-f(\bm x)}{h} -\frac{f(\bm x+h'\bm e_i)-f(\bm x)}{h'} \right)^2\right] \\
    &= \E\left[\left( \frac{f(\bm x+h\bm e_i)-f(\bm x)}{h}\right)^2\right] - 2\E\left[\left(\frac{f(\bm x+h\bm e_i)-f(\bm x)}{h} \right)\left(\frac{f(\bm x+h'\bm e_i)-f(\bm x)}{h'}\right)\right] \\
    &\quad+ \E\left[\left( \frac{f(\bm x+h'\bm e_i)-f(\bm x)}{h'}\right)^2\right] \\
    &\to \partial^{\bm e_i,\bm e_i}k(\bm x,\bm x) -2\partial^{\bm e_i,\bm e_i}k(\bm x,\bm x)+\partial^{\bm e_i,\bm e_i}k(\bm x,\bm x) = 0
\end{aligned}
\end{equation}
where the limit is taken as $h,h'\to 0$. \sloppy Since $L^2$ is complete, Equation (\ref{eq: ms convergence}) shows that $\lim_{h\to 0} \frac{f(\bm x+h\bm e_i)-f(\bm x)}{h}$ exists in $L^2$, i.e.~that $\partial_{ms}^{\bm e_i}f(\bm x)$ exists. Moreover Equation (\ref{eq: ms to kernel}) then implies that $\E\left[\partial^{\bm e_i}_{ms}f(\bm x)\partial^{\bm e_i}_{ms}f(\bm y)\right] =\partial^{\bm e_i,\bm e_i}k(\bm x,\bm y)$ for $\bm x,\bm y\in O$ close enough. Hence
$$
\begin{aligned}
& \E\left[(\partial^{\bm e_i}_{ms}f(\bm x+\bm h)-\partial^{\bm e_i}_{ms}f(\bm x))^2\right] \\
& =\E\left[\partial^{\bm e_i}_{ms}f(\bm x+\bm h)^2\right]-2\E\left[\partial^{\bm e_i}_{ms}f(\bm x+\bm h)\partial^{\bm e_i}_{ms}f(\bm x)\right]+\E\left[\partial^{\bm e_i}_{ms}f(\bm x)^2\right] \\
&= \partial^{\bm e_i,\bm e_i}k(\bm x+\bm h,\bm x+\bm h)-2\partial^{\bm e_i,\bm e_i}k(\bm x+\bm h,\bm x)+\partial^{\bm e_i,\bm e_i}k(\bm x,\bm x) \\
&\to 0
\end{aligned}
$$
as $\bm h\to \bm 0$ since $\partial^{\bm e_i,\bm e_i}k$ is continuous. So $\partial^{\bm e_i}_{ms}f$ is mean-square continuous.

\textbf{\ref{lem extrapol: ms dve} $\Rightarrow$ \ref{lem extrapol: all}:} For $\bm x,\bm y\in O$ and $1\leq i\leq d$ we have
$$\frac{k(\bm x+ h\bm e_i, \bm y)-k(\bm x,\bm y)}{h} = \E\left[ \frac{f(\bm x+h \bm e_i)-f(\bm x)}{h}f(\bm y) \right] \to \E\left[\partial^{\bm e_i}_{ms}f(\bm x)f(\bm y)\right]$$
as $h\to 0$. Hence $\partial^{\bm e_i,\bm 0}k(\bm x,\bm y)$ exists. Furthermore, for $1\leq j\leq d$,
    \begin{equation}\label{eq: kernel mixed dve}
    \begin{aligned}
    \frac{\partial^{\bm e_i,\bm 0}k(\bm x, \bm y+h\bm e_j)-\partial^{\bm e_i,\bm 0}k(\bm x,\bm y)}{h}&=\E\left[ \partial^{\bm e_i}_{ms}f(\bm x) \frac{f(\bm y+h\bm e_j)-f(\bm y)}{h}\right] \\
    &\to \E\left[ \partial^{\bm e_i}_{ms}f(\bm x)\partial^{\bm e_j}_{ms}f(\bm y)\right]
    \end{aligned}
    \end{equation}
as $h\to 0$. Hence $\partial^{\bm e_i,\bm e_j}k$ exists, and is continuous by Equation (\ref{eq: kernel mixed dve}), since $\partial^{\bm e_i}_{ms}f$ and $\partial^{\bm e_j}_{ms}f$ are mean-square continuous.

Now it remains to show that in all cases $\partial^{\bm e_i}_{ms}f \sim \GP(0,\partial^{\bm e_i,\bm e_i}k)$. Let $\{\bm x_1,\dots,\bm x_N\}\subset O$ a finite set of points. Then we need to show
    \begin{equation}\label{eq: gauss distribution derivative}
    (\partial^{\bm e_i}_{ms}f(\bm x_1),\dots,\partial^{\bm e_i}_{ms}f(\bm x_N))\sim \mathcal N \left(\bm 0, \left(\partial^{\bm e_i,\bm e_i}k(\bm x_p,\bm x_q)\right)_{p,q=1}^N\right).
    \end{equation}
    $(\partial^{\bm e_i}_{ms}f(\bm x_1),\dots,\partial^{\bm e_i}_{ms}f(\bm x_N))$ is the limit in $L^2(\Omega;\R^N)$ as $h\to 0$ of
    \begin{equation}\label{eq: gauss distribution difference}
    \begin{aligned}
    &\left(\frac{f(\bm x_p+h\bm e_i)-f(\bm x_p)}{h}\right)_{p=1}^N \\
    \sim \mathcal N \bigg(\bm 0,\bigg(&\frac{k(\bm x_p+h\bm e_i,\bm x_q+h\bm e_i)-k(\bm x_p+h\bm e_i,\bm x_q)-k(\bm x_p,\bm x_q +h\bm e_i)+k(\bm x_p,\bm x_q)}{h^2}\bigg)_{p,q=1}^N\bigg).
    \end{aligned}
    \end{equation}
    Now convergence in $L^2$ implies convergence in distribution, and we see that the multivariate normal distribution in Equation (\ref{eq: gauss distribution difference}) converges to the multivariate normal distribution in Equation (\ref{eq: gauss distribution derivative}) as $h\to 0$. So this shows $\partial^{\bm e_i}_{ms}f \sim \GP(0,\partial^{\bm e_i,\bm e_i}k)$.
\end{proof}
We can now prove our main result, Theorem \ref{thm: holder}, which we restate here for convenience:
\thmholder*
\begin{proof}\label{proof: mainthm}
    We only prove the general case \ref{thm holder: general case}.
    Cases \ref{thm holder: stationary case} and \ref{thm holder: isotropic case} then follow by taking intersections over $\epsilon \in (0,\gamma)$ in Lemma \ref{lem equiv}.

    \textbf{$\Longleftarrow$:} For $n= 0$ the result follows from applying the Kolmogorov continuity theorem to GPs, see \cite[Theorem 2.14]{nummi_necessary_2024} or \cite[Page 347]{potthoff_sample_2009}.

    For $n = 1$, the existence and continuity of $\partial^{\bm e_i,\bm e_i}k$ implies by Lemma \ref{lem extrapol} \ref{lem extrapol: all} $\Rightarrow$ \ref{lem extrapol: ms dve} that $f$ is mean-square differentiable in direction $i$. Moreover this mean-square partial derivative satisfies $\partial^{\bm e_i}_{ms}f\sim\GP(0,\partial^{\bm e_i,\bm e_i}k)$. Now applying the case $n= 0$ to $\partial^{\bm e_i,\bm e_i}k$ we deduce that $\partial^{\bm e_i}_{ms}f$ has samples in\footnote{The mean-square partial derivative $\partial^{\bm e_i}_{ms}f(\bm x)$ at a point $\bm x\in O$ is only well-defined almost surely on $\Omega$. But we still talk of sample paths for the process $\partial^{\bm e_i}_{ms}f$, since we mean this up to modification, as described in Section \ref{sec: preliminaries}.} $C^{\gamma^-}_{\loc}(O)$. In particular it has continuous samples, and this implies by \cite[Theorem 3.2]{potthoff_sample_2010} that $f$ has differentiable samples in direction $i$, with $\partial^{\bm e_i}f(\bm x) = \partial^{\bm e_i}_{ms}f(\bm x)$ almost surely, for all $\bm x \in O$. Thus $\partial^{\bm e_i}f\sim\GP(0,\partial^{\bm e_i,\bm e_i}k)$ and $\partial^{\bm e_i}f$ has samples in $C^{\gamma^-}_{\loc}(O)$. $1\leq i\leq d$ was arbitrary, so $f$ has samples in $C^{(1+\gamma)^-}_{\loc}(O)$.
    
    For $n >1$, we apply the same argument inductively on the partial derivatives.

    \textbf{$\Longrightarrow$:} For $n = 0$, \cite[Theorem 2.14]{nummi_necessary_2024} gives the converse to Kolmogorov's theorem for GPs (or, more generally, certain hypercontractive processes).

    Now suppose $n = 1$. Pick $1\leq i\leq d$. For any finite set of points $\{\bm x_1,\dots,\bm x_N\}\subset O$, $(\partial^{\bm e_i}f(\bm x_1),\dots,\partial^{\bm e_i}f(\bm x_N))$ is the almost sure limit of centered multivariate Gaussians with distribution
    $$\mathcal N \bigg(\bm 0,\bigg(\frac{k(\bm x_p+h\bm e_i,\bm x_q+h\bm e_i)-k(\bm x_p+h\bm e_i,\bm x_q)-k(\bm x_p,\bm x_q +h\bm e_i)+k(\bm x_p,\bm x_q)}{h^2}\bigg)_{p,q=1}^N\bigg)$$
    (see Equation (\ref{eq: gauss distribution difference})), so is itself a centered multivariate Gaussian distribution. Hence $\partial^{\bm e_i}f$ is a GP.

    Let $\bm x, \bm x+h\bm e_i\in O$. By the mean value theorem, for each $\omega\in \Omega$ there is $\xi_{\bm x,\omega}(h)$ between $\bm x$ and $\bm x+h\bm e_i$ such that
    $$\frac{f(\bm x+h\bm e_i, \omega)-f(\bm x, \omega)}{h} = \partial^{\bm e_i}f(\xi_{\bm x,\omega}(h), \omega).$$
    Thus
    $$\left|\frac{f(\bm x+h\bm e_i, \omega)-f(\bm x, \omega)}{h}-\partial^{\bm e_i}f(\bm x, \omega)\right|^2= |\partial^{\bm e_i}f(\xi_{\bm x,\omega}(h), \omega)-\partial^{\bm e_i}f(\bm x, \omega)|^2 \leq C_\omega^2|h|^{2\epsilon}$$
    for all $\epsilon\in(0,\gamma)$, for some constant $C_\omega$ depending on $\omega$ but not on $h$, assuming $h$ is small enough. Thus we have
    $$
    \E\left[\left(\frac{f(\bm x+h\bm e_i)-f(\bm x)}{h}-\partial^{\bm e_i}f(\bm x)\right)^2\right] \leq \underbrace{\E\left[C_\omega^2\right]}_{(*)}\cdot\underbrace{|h|^{2\epsilon}}_{(**)}
    $$
    for $h$ small enough and $\epsilon \in (0,\gamma)$. $(*) <\infty$ by \cite[Theorem 2.14]{nummi_necessary_2024}, which we can apply since we showed that $\partial^{\bm e_i}f$ is a GP. Also $(**)\to 0$ as $h\to 0$. Hence
    $$\frac{f(\bm x+h\bm e_i)-f(\bm x)}{h} \xrightarrow{L^2} \partial^{\bm e_i}f(\bm x)$$
    as $h\to 0$, i.e.~$f$ is mean-square differentiable in direction $i$. Hence by Lemma \ref{lem extrapol} \ref{lem extrapol: ms dve} $\Rightarrow$ \ref{lem extrapol: all}, $\partial^{\bm e_i,\bm e_j}k$ exists and is continuous for all $1\leq i,j\leq d$, and $\partial^{\bm e_i}f\sim \GP(0,\partial^{\bm e_i,\bm e_i}k)$. Now
    \begin{equation}\label{eq: holder kernel}
    \begin{aligned}
    &\partial^{\bm e_i,\bm e_j}k(\bm x+\bm h,\bm x+\bm h)-\partial^{\bm e_i,\bm e_j}k(\bm x+\bm h,\bm x)-\partial^{\bm e_i,\bm e_j}k(\bm x,\bm x+\bm h)+\partial^{\bm e_i,\bm e_j}k(\bm x,\bm x) \\
    &= \E\left[(\partial^{\bm e_i}f(\bm x+\bm h)-\partial^{\bm e_i}f(\bm x))(\partial^{\bm e_j}f(\bm x+\bm h)-\partial^{\bm e_j}f(\bm x))\right] \\
    &\leq \E\left[(\partial^{\bm e_i}f(\bm x+\bm h)-\partial^{\bm e_i}f(\bm x))^2\right]^{\nicefrac{1}{2}}\E\left[(\partial^{\bm e_j}f(\bm x+\bm h)-\partial^{\bm e_j}f(\bm x))^2\right]^{\nicefrac{1}{2}} \\
    &\leq \underbrace{\E\left[A_\omega^2\right]^{\nicefrac{1}{2}}}_{<\infty}\|\bm h\|^\epsilon\cdot\underbrace{\E\left[B_\omega^2\right]^{\nicefrac{1}{2}}}_{<\infty}\|\bm h\|^\epsilon = \mathcal O(\|\bm h\|^{2\epsilon})
    \end{aligned}
    \end{equation}
    as $\bm h \to\bm 0$, locally uniformly in $\bm x\in O$, for all $\epsilon\in (0,\gamma)$, where $A_\omega$, $B_\omega$ are some constants depending on $\omega$, by \cite[Theorem 2.14]{nummi_necessary_2024} applied to $\partial^{\bm e_i}f$ and $\partial^{\bm e_j}f$, for $1\leq i,j\leq d$.

    Finally, for $n>1$ we apply the same argument inductively on the partial derivatives.
\end{proof}
\section{Proofs in Examples}\label{app: proofs}
\subsection{Proof of Proposition \ref{prop: matern}}\label{proof: matern}
\begin{proof}
    For $\rho >0$ we can write for $\nu\not\in\N$
    $$K_\nu(\rho) = \frac{\pi}{2}\frac{I_{-\nu}(\rho)-I_\nu(\rho)}{\sin(\nu\pi)}$$
    where
    $$I_{\pm\nu}(\rho) = \left(\frac{\rho}{2}\right)^{\pm\nu}\sum_{j=0}^\infty \frac{1}{j!\,\Gamma(j\pm\nu+1)}\left(\frac{\rho}{2}\right)^{2j} $$
    are modified Bessel functions of the first kind \cite[Equations 9.6.2 \& 9.6.10]{abramowitz_handbook_1965}. So we can write the Matérn kernel (\ref{eq: matern kernel}) as
    $$k_r(x) = C_\nu \bigg( \underbrace{\sum_{j=0}^\infty \frac{1}{j!\,\Gamma(j-\nu+1)}\left(\frac{x}{2}\right)^{2j}}_{(*)}- \frac{|x|^{2\nu}}{2^\nu}\underbrace{\sum_{j=0}^\infty \frac{1}{j!\,\Gamma(j+\nu+1)}\left(\frac{x}{2}\right)^{2j}}_{(**)}\bigg)$$
    for $x\in\R$, where $C_\nu>0$ is some constant. The power series $(*)$ and $(**)$ are smooth in $x$, and $|x|^{2\nu}$ has precisely $\lceil 2\nu-1\rceil$  continuous derivatives with $(2\nu-2\lceil\nu-1\rceil)$-Hölder continuous $2\lceil\nu-1\rceil$\textsuperscript{th} derivative at 0. So, for $\nu\not\in\N$, the result follows by \thmitemref{Theorem}{thm: holder}{thm holder: isotropic case}.

    For $\nu=n\in\N$, we instead use the following formula for the modified Bessel function of the second kind \cite[Equations 9.6.11]{abramowitz_handbook_1965}: for $\rho>0$
    $$
    \begin{aligned}
    K_n(\rho)&=\frac{1}{2}\left(\frac{\rho}{2}\right)^{-n} \sum_{j=0}^{n-1} \frac{(n-j-1) !}{j !}\left(-\frac{\rho^2}{4} \right)^j+(-1)^{n+1} \log \left(\frac{\rho}{2} \right) I_n(\rho) \\
    &+(-1)^n \frac{1}{2}\left(\frac{\rho}{2}\right)^n \sum_{j=0}^{\infty}\left(\psi(j+1)+\psi(n+j+1)\right) \frac{\left(\frac{\rho^2}{4}\right)^j}{j !(n+j) !}
    \end{aligned}
    $$
    where $\psi$ is the digamma function. So we can write the Matérn kernel (\ref{eq: matern kernel}) as
    $$
    \begin{aligned}
    k_r(x) &= A_n\underbrace{\sum_{j=0}^{n-1} \frac{(n-j-1) !}{j !}(-1)^j\left(-\frac{x^2}{4} \right)^j}_{(I)}+B_n\underbrace{x^{2n}\log \left(\frac{x}{2} \right)}_{(II)} \underbrace{\sum_{j=0}^\infty \frac{1}{j!(n+j)!}\left(\frac{x}{2}\right)^{2j}}_{(III)} \\
    &+C_n\underbrace{x^{2n} \sum_{j=0}^{\infty}\left(\psi(j+1)+\psi(n+j+1)\right) \frac{\left(\frac{x}{2}\right)^{2j}}{j !(n+j) !}}_{(IV)}
    \end{aligned}
    $$
    for some constants $A_n,B_n,C_n\neq 0$. $(I)$ and $(IV)$ are smooth in $x$. $(III)$ is smooth in $x$ and is non-zero at $x=0$. $(II)$ has precisely $2n-1$ continuous derivatives, with almost-$2$-Hölder continuous $(2n-2)$\textsuperscript{th} derivative at 0 (precisely, a Hölder decay of $\mathcal O(h^2\log h)$).
    So the result follows by \thmitemref{Theorem}{thm: holder}{thm holder: isotropic case}.
\end{proof}
\subsection{Proof of Proposition \ref{prop: wendland}}\label{proof: wendland}
\begin{proof}
    It is shown in \cite[Theorem 9.12]{wendland_scattered_2004} that the Wendland kernels may be written as
    $$k_r(x) = \sum_{j=0}^{\lfloor \nicefrac{d}{2}\rfloor +3n+1}d^{(\lfloor \nicefrac{d}{2}\rfloor +n+1)}_{j,n}|x|^j$$
    for $x\in\R$, where $d^{(\lfloor \nicefrac{d}{2}\rfloor +n+1)}_{j,n}\in\R$ are coefficients. Furthermore the odd degree coefficients satisfy
    $d^{(\lfloor \nicefrac{d}{2}\rfloor +n+1)}_{2j+1,n} = 0$ if and only if $0\leq j\leq 2n-1$. Now for all $j\in \N_0$, $|x|^{2j}$ is smooth and $|x|^{2j+1}$ is precisely $2n$ times continuously differentiable with Lipschitz $2n$\textsuperscript{th} derivative. So the result follows by \thmitemref{Theorem}{thm: holder}{thm holder: isotropic case}.
\end{proof}
\subsection{Proof of Proposition \ref{prop: feature}}\label{proof: feature}
\begin{proof}
    This follows from \thmitemref{Theorem}{thm: holder}{thm holder: general case}, by noting that $$\partial^{\bm\alpha,\bm\beta}k(\bm x,\bm y) = \partial^{\bm\alpha}\bm \phi(\bm x)^T\partial^{\bm\beta}\bm \phi(\bm y)$$
    for all $\bm x,\bm y\in O$ and $\bm \alpha, \bm\beta\in\N_0^d$ with $|\bm \alpha|,|\bm \beta|\leq n$, and moreover when $|\bm \alpha|=|\bm \beta|=n$,
    $$
    \begin{aligned}
    &|\partial^{\bm\alpha,\bm\beta}k(\bm x+\bm h,\bm x+\bm h)-\partial^{\bm\alpha,\bm\beta}k(\bm x+\bm h,\bm x)-\partial^{\bm\alpha,\bm\beta}k(\bm x,\bm x+\bm h)+\partial^{\bm\alpha,\bm\beta}k(\bm x,\bm x)| \\
    &= (\partial^{\bm\alpha}\bm \phi(\bm x+\bm h)-\partial^{\bm\alpha}\bm\phi(\bm x))^T(\partial^{\bm\beta}\bm \phi(\bm x+\bm h)-\partial^{\bm\beta}\bm\phi(\bm x))\\
    &\leq \left(\sum_{i=1}^m|\partial^{\bm\alpha}\phi_i(\bm x+\bm h)-\partial^{\bm\alpha}\phi_i(\bm x)|^2\right)^{\nicefrac{1}{2}}\left(\sum_{i=1}^m|\partial^{\bm\beta}\phi_i(\bm x+\bm h)-\partial^{\bm\beta}\phi_i(\bm x)|^2\right)^{\nicefrac{1}{2}}\\
    &=\mathcal O(\|\bm h\|^{\bm \epsilon})\cdot \mathcal O(\|\bm h\|^{\bm \epsilon}) = \mathcal O(\|\bm h\|^{2\bm \epsilon})
    \end{aligned}
    $$
    as $\bm h\to \bm 0$, locally uniformly in $\bm x\in O$, for all $\epsilon \in (0,\gamma)$, where we used the Cauchy-Schwarz inequality.
\end{proof}
\subsection{Proof of Proposition \ref{prop: product}}\label{proof: product}
\begin{proof}
    We assume $m=2$; the proof for general $m$ can then be done inductively.
    Note that by the product rule
    $$\partial^{\bm\alpha,\bm\beta}k(\bm x,\bm y)= \sum_{\bm\gamma_1+\bm\gamma_2=\bm\alpha}\sum_{\bm\delta_1+\bm\delta_2=\bm\beta}\binom{|\bm \alpha|}{|\bm \gamma_1|}\binom{|\bm \beta|}{|\bm\delta_1|}\partial^{\bm\gamma_1,\bm\delta_1}k_1(\bm x,\bm y)\partial^{\bm\gamma_2,\bm\delta_2}k_2(\bm x,\bm y)$$
    for all $\bm x,\bm y\in O$ and $\bm \alpha,\bm\beta\in\N_0^d$ with $|\bm \alpha|,|\bm\beta|\leq n$.

    Now for $|\bm \alpha|= |\bm \beta| = n$,
    \begin{equation}\label{eq: diffted product expansion}
    \begin{aligned}
    &\partial^{\bm\alpha,\bm\beta}k(\bm x+\bm h,\bm x+\bm h)-\partial^{\bm\alpha,\bm\beta}k(\bm x+\bm h,\bm x)-\partial^{\bm\alpha,\bm\beta}k(\bm x,\bm x+\bm h)+\partial^{\bm\alpha,\bm\beta}k(\bm x,\bm x) \\
    &= \sum_{\bm\gamma_1+\bm\gamma_2=\bm\alpha}\sum_{\bm\delta_1+\bm\delta_2=\bm\beta}\binom{|\bm \alpha|}{|\bm \gamma_1|}\binom{|\bm \beta|}{|\bm\delta_1|} \Big(\partial^{\bm\gamma_1,\bm\delta_1}k_1(\bm x+\bm h,\bm x+\bm h)\partial^{\bm\gamma_2,\bm\delta_2}k_2(\bm x+\bm h,\bm x+\bm h) \\
    &\quad -\partial^{\bm\gamma_1,\bm\delta_1}k_1(\bm x+\bm h,\bm x)\partial^{\bm\gamma_2,\bm\delta_2}k_2(\bm x+\bm h,\bm x)-\partial^{\bm\gamma_1,\bm\delta_1}k_1(\bm x,\bm x+\bm h)\partial^{\bm\gamma_2,\bm\delta_2}k_2(\bm x,\bm x+\bm h)\\
    &\quad +\partial^{\bm\gamma_1,\bm\delta_1}k_1(\bm x,\bm x)\partial^{\bm\gamma_2,\bm\delta_2}k_2(\bm x,\bm x)\Big).
    \end{aligned}
    \end{equation}
    To show the expression in Equation (\ref{eq: diffted product expansion}) is $\mathcal O(\|\bm h\|^{2\epsilon})$, it is therefore sufficient to show that
    \begin{equation}\label{eq: holder cond product}
    \begin{aligned}
    &|\partial^{\bm\gamma_1,\bm\delta_1}k_1(\bm x+\bm h,\bm x+\bm h)\partial^{\bm\gamma_2,\bm\delta_2}k_2(\bm x+\bm h,\bm x+\bm h)-\partial^{\bm\gamma_1,\bm\delta_1}k_1(\bm x+\bm h,\bm x)\partial^{\bm\gamma_2,\bm\delta_2}k_2(\bm x+\bm h,\bm x) \\
    &\quad-\partial^{\bm\gamma_1,\bm\delta_1}k_1(\bm x,\bm x+\bm h)\partial^{\bm\gamma_2,\bm\delta_2}k_2(\bm x,\bm x+\bm h)+\partial^{\bm\gamma_1,\bm\delta_1}k_1(\bm x,\bm x)\partial^{\bm\gamma_2,\bm\delta_2}k_2(\bm x,\bm x)|= \mathcal O(\|\bm h\|^{2\epsilon})
    \end{aligned}
    \end{equation}
    as $\bm h\to \bm 0$, locally uniformly in $\bm x\in O$, for all $\epsilon\in (0,\gamma)$, $\bm\gamma_1+\bm\gamma_2=\bm\alpha$ and $\bm\delta_1+\bm\delta_2=\bm\beta$. The inside of the absolute value on the left hand side of Equation (\ref{eq: holder cond product}) can be written as
    \begin{equation}\label{eq: product expansion}
    \begin{aligned}
    &\partial^{\bm\gamma_1,\bm\delta_1}k_1(\bm x+\bm h,\bm x+\bm h)\partial^{\bm\gamma_2,\bm\delta_2}k_2(\bm x+\bm h,\bm x+\bm h)-\partial^{\bm\gamma_1,\bm\delta_1}k_1(\bm x+\bm h,\bm x)\partial^{\bm\gamma_2,\bm\delta_2}k_2(\bm x+\bm h,\bm x) \\
    &\quad-\partial^{\bm\gamma_1,\bm\delta_1}k_1(\bm x,\bm x+\bm h)\partial^{\bm\gamma_2,\bm\delta_2}k_2(\bm x,\bm x+\bm h)+\partial^{\bm\gamma_1,\bm\delta_1}k_1(\bm x,\bm x)\partial^{\bm\gamma_2,\bm\delta_2}k_2(\bm x,\bm x) \\
    &= \partial^{\bm\gamma_2,\bm\delta_2}k_2(\bm x+\bm h,\bm x+\bm h)\\
    &\quad\cdot\underbrace{(\partial^{\bm\gamma_1,\bm\delta_1}k_1(\bm x+\bm h,\bm x+\bm h)-\partial^{\bm\gamma_1,\bm\delta_1}k_1(\bm x+\bm h,\bm x)-\partial^{\bm\gamma_1,\bm\delta_1}k_1(\bm x,\bm x+\bm h)+\partial^{\bm\gamma_1,\bm\delta_1}k_1(\bm x,\bm x))}_{(*)} \\
    &\quad+\partial^{\bm\gamma_1,\bm\delta_1}k_1(\bm x,\bm x) \\
    &\quad\cdot\underbrace{(\partial^{\bm\gamma_2,\bm\delta_2}k_2(\bm x+\bm h,\bm x+\bm h)-\partial^{\bm\gamma_2,\bm\delta_2}k_2(\bm x+\bm h,\bm x)-\partial^{\bm\gamma_2,\bm\delta_2}k_2(\bm x,\bm x+\bm h)+\partial^{\bm\gamma_2,\bm\delta_2}k_2(\bm x,\bm x))}_{(**)} \\
    &\quad+\underbrace{(\partial^{\bm\gamma_1,\bm\delta_1}k_1(\bm x+\bm h,\bm x)-\partial^{\bm\gamma_1,\bm\delta_1}k_1(\bm x,\bm x))}_{(I)}\underbrace{(\partial^{\bm\gamma_2,\bm\delta_2}k_2(\bm x+\bm h,\bm x+\bm h)-\partial^{\bm\gamma_2,\bm\delta_2}k_2(\bm x+\bm h,\bm x))}_{(II)} \\
    &\quad+\underbrace{(\partial^{\bm\gamma_1,\bm\delta_1}k_1(\bm x,\bm x+\bm h)-\partial^{\bm\gamma_1,\bm\delta_1}k_1(\bm x,\bm x))}_{(III)}\underbrace{(\partial^{\bm\gamma_2,\bm\delta_2}k_2(\bm x+\bm h,\bm x+\bm h)-\partial^{\bm\gamma_2,\bm\delta_2}k_2(\bm x,\bm x+\bm h))}_{(IV)}
    \end{aligned}
    \end{equation}
    By arguing as in the proof of Theorem \ref{thm: holder}, Equation (\ref{eq: holder kernel}) we have $(*),(**)=\mathcal O(\|\bm h\|^{2\epsilon})$. Moreover, by the Cauchy-Schwarz inequality,
    \begin{equation}\label{eq: cauchy schwarz application}
    \begin{aligned}
        |(I)| &= \left|\E\big[\partial^{\bm\gamma_1}f_1(\bm x+\bm h)\partial^{\bm\delta_1}f_1(\bm x)-\partial^{\bm\gamma_1}f_1(\bm x)\partial^{\bm\delta_1}f_1(\bm x)\big]\right| \\
        &\leq \E\big[\partial^{\bm\delta_1}f_1(\bm x)^2\big]^{\nicefrac{1}{2}}\E\left[(\partial^{\bm\gamma_1}f_1(\bm x+\bm h)-\partial^{\bm\gamma_1}f_1(\bm x))^2\right]^{\nicefrac{1}{2}} \\
        &= \partial^{\bm\delta_1,\bm\delta_1}k_1(\bm x,\bm x)^{\nicefrac{1}{2}}(\partial^{\bm\gamma_1,\bm\gamma_1}k_1(\bm x+\bm h,\bm x+\bm h)-\partial^{\bm\gamma_1,\bm\gamma_1}k_1(\bm x+\bm h,\bm x) \\
        &\quad-\partial^{\bm\gamma_1,\bm\gamma_1}k_1(\bm x,\bm x+\bm h)+\partial^{\bm\gamma_1,\bm\gamma_1}k_1(\bm x,\bm x))^{\nicefrac{1}{2}} \\
        &= \mathcal O(\|\bm h\|^{\epsilon})
    \end{aligned}
    \end{equation}
    where $f_1\sim \GP(0,k_1)$. Similarly $(II),(III),(IV) = \mathcal O(\|\bm h\|^{\epsilon})$. By Equation (\ref{eq: product expansion}) we therefore deduce Equation (\ref{eq: holder cond product}), which concludes the proof.
\end{proof}